\pdfoutput=1 

\documentclass[twoside]{article}

\usepackage[accepted]{aistats2020_bis}
\usepackage{color,xcolor}
\usepackage[english]{babel}
\usepackage[utf8]{inputenc}
\usepackage{amsmath,amssymb,amsthm, amsfonts, mathtools}
\usepackage[colorinlistoftodos]{todonotes}
\usepackage{microtype}
\usepackage{graphicx}
\usepackage{booktabs}
\usepackage[inline, shortlabels]{enumitem}
\definecolor{mydarkblue}{rgb}{0,0.08,0.45}
\usepackage[colorlinks=true,
    linkcolor=mydarkblue,
    citecolor=mydarkblue,
    filecolor=mydarkblue,
    urlcolor=mydarkblue,
    breaklinks=true,
    ]{hyperref}
\usepackage{url}
\usepackage{nicefrac}
\usepackage{microtype}
\usepackage{array}
\usepackage{verbatim}
\usepackage{subcaption}
\newcolumntype{P}[1]{>{\centering\arraybackslash}p{#1}}

\newcommand{\N}{\mathcal N}

\newcommand{\R}{\mathbb{R}}
\newcommand{\C}{\mathbb{C}}
\newcommand{\bigO}{\mathcal{O}}

\newcommand{\half}{\frac{1}{2}}

\DeclareMathOperator*{\argmin}{arg\,min}
\DeclareMathOperator*{\Sp}{Sp}
\DeclareMathOperator*{\diag}{diag}
\DeclareMathOperator*{\Tr}{Tr}
\DeclareMathOperator*{\Id}{Id}
\DeclareMathOperator*{\Ker}{Ker}
\DeclareMathOperator*{\sign}{sign}

\DeclarePairedDelimiterX{\inp}[2]{\langle}{\rangle}{#1, #2}

\usepackage{thmtools}
\usepackage{thm-restate}

\declaretheorem[name=Theorem]{thm}
\declaretheorem[name=Theorem, numbered=no]{thm*}
\declaretheorem[name=Proposition]{prop}
\declaretheorem[name=Lemma]{lemma}
\declaretheorem[name=Corollary]{cor}
\declaretheorem[name=Definition, style=definition]{definition}

\declaretheorem[name=Example, style=definition]{ex}
\declaretheorem[name=Remark, style=remark]{rmk}
\declaretheorem[name=Assumption, style=definition]{asm}

\usepackage{pgfplotstable} %
\usetikzlibrary{automata, positioning, arrows, shapes, fit, calc, intersections}
\usepgfplotslibrary{statistics}

\usepgfplotslibrary{fillbetween} %
\newcommand{\drawSpectrum}[1][\textwidth]{
\newcounter{t}
\setcounter{t}{80}
\newcounter{L}
\setcounter{L}{10}
\newcounter{g}
\setcounter{g}{4}
\begin{tikzpicture}
\begin{axis}[
    ticks=none,
    grid=major,
    axis equal image,
    max space between ticks=1000pt,
    width = #1,
    yticklabel={
    $\pgfmathprintnumber{\tick}i$
    },
	xmin=0,   xmax={\value{L} + 2},
	ymin={-\value{L} - 0.2},   ymax={\value{L} + 0.2},
	]
    
   \draw [help lines, name path=Lpath, domain=-\value{t}:\value{t}, draw=\spectrumcolor, thick, samples=65] plot (axis cs: {\value{L}*cos(\x)}, {\value{L}*sin(\x)});
   \draw [name path=gammapath, domain=-64.27066043:64.27066043, draw=\spectrumcolor, thick, samples=65] plot (axis cs: {\value{g}*cos(\x)}, {\value{g}*sin(\x)});
   \draw [draw=\spectrumcolor, thick, dashed](axis cs: {\value{L}*cos(\value{t})}, {\value{g}*sin(64.27066043)}) -- (axis cs: {\value{L}*cos(\value{t})}, {\value{g}*sin(-64.27066043)});
   \draw [draw=\spectrumcolor, thick](axis cs: {\value{L}*cos(\value{t})}, {\value{g}*sin(64.27066043)}) -- (axis cs: {\value{L}*cos(\value{t})}, {\value{L}*sin(\value{t})});
   \draw [draw=\spectrumcolor, thick](axis cs: {\value{L}*cos(\value{t})}, {\value{g}*sin(-64.27066043)}) -- (axis cs: {\value{L}*cos(\value{t})}, {\value{L}*sin(-\value{t})});
   
    \addplot [
        thick,
        color=blue,
        fill=\spectrumcolor!50!white,
        fill opacity=0.3,
    ]
    fill between[
        of= Lpath and gammapath,
    ];
   
   \node[anchor=south] (mu) at (axis cs: {\value{L} * cos(\value{t})+.2}, 0.25) {$\min \Re \lambda$};
   \draw [fill=black] (axis cs: {\value{L} * cos(\value{t})}, 0) circle (2pt);
   \node [anchor=north] (g) at (axis cs: {\value{g}}, -0.25) {$\min |\lambda|$};
   \draw [fill=black] (axis cs: {\value{g}}, 0) circle (2pt);
   \node [anchor=north] (L) at (axis cs: {\value{L}-0.25}, -0.25) {$\max |\lambda|$};
   \draw [fill=black] (axis cs: {\value{L}}, 0) circle (2pt);
   \draw [fill=\spectrumcolor] (axis cs: {\value{L}*0.707}, {\value{L}*0.707}) circle (2pt) ;
   \draw [fill=\spectrumcolor] (axis cs: {\value{L}*0.707}, {-\value{L}*0.707}) circle (2pt) ;
   \draw [fill=\spectrumcolor] (axis cs: {\value{L}*0.5}, 0) circle (2pt);
   \draw [fill=\spectrumcolor] (axis cs: {\value{L} * cos(\value{t})}, 5+\value{g}) circle (2pt);
   \draw [fill=\spectrumcolor] (axis cs: {\value{L} * cos(\value{t})}, -5-\value{g}) circle (2pt);
   \draw [fill=\spectrumcolor] (axis cs: {\value{g}+2}, \value{g}-0.75) circle (2pt);
   \draw [fill=\spectrumcolor] (axis cs: {\value{g}+2}, -\value{g}+0.75) circle (2pt);
   \draw [fill=\spectrumcolor] (axis cs: {\value{g}*0.707}, {\value{g}*0.707}) circle (2pt);
   \draw [fill=\spectrumcolor] (axis cs: {\value{g}*0.707}, {-\value{g}*0.707}) circle (2pt);
   \end{axis}
\end{tikzpicture}
}

\usepackage[round]{natbib}

\usepackage{sidecap}
\usepackage[capitalize]{cleveref}
\crefname{thm}{Thm.}{theorems}
\crefname{fig}{Fig.}{fig.}
\crefname{prop}{Prop.}{propositions}
\crefname{cor}{Cor.}{corollaries}
\crefname{lemma}{Lem.}{lem.}
\crefname{asm}{Assumption}{assumptions}
\crefname{ex}{Example}{examples}
\crefname{rmk}{Remark}{examples}
\crefformat{subsection}{\S#2#1#3}
\crefformat{section}{\S#2#1#3}
\crefformat{equation}{(#2\textup{#1}#3)}

\newif\ifcompress
\compresstrue

\newcommand{\varvspace}[1]{
\ifcompress
    \vspace{#1}
\fi
}

\begin{document}

\runningtitle{A Unified Analysis of Gradient-Based Methods for a Whole Spectrum of Games}
\twocolumn[

\aistatstitle{A Tight and Unified Analysis of Gradient-Based Methods \\ for a Whole Spectrum of Differentiable Games}

\aistatsauthor{ Waïss Azizian$^{1,\dag}$ \; Ioannis Mitliagkas$^{2,\ddagger}$ \;  Simon Lacoste-Julien$^{2,\ddagger}$ \;  Gauthier Gidel$^{2}$}
\runningauthor{Waïss Azizian, Ioannis Mitliagkas, Simon Lacoste-Julien, Gauthier Gidel}
\aistatsaddress{ $^1$École Normale Supérieure, Paris  $\quad ^2$Mila \& DIRO, Université de Montréal} 

]

\begin{abstract}
We consider differentiable games where the goal is to find a Nash equilibrium. 
The machine learning community has recently started using variants of the gradient method (\textsc{GD}).
Prime examples are extragradient (\textsc{EG}), the optimistic gradient method (\textsc{OG}) and consensus optimization (\textsc{CO}), which enjoy linear convergence in cases like bilinear games, where the standard \textsc{GD} fails. 
The full benefits of theses relatively new methods are not known as there is no unified analysis for both strongly monotone and bilinear games.
We provide new analyses of the \textsc{EG}’s local and global convergence properties and use is to get a tighter global convergence rate for \textsc{OG} and \textsc{CO}.
Our analysis covers the whole range of settings between bilinear and strongly monotone games.
It reveals that these methods converge via different mechanisms at these extremes;
in between, it exploits the most favorable mechanism for the given problem.
We then prove that \textsc{EG} achieves the optimal rate for a wide class of algorithms with any number of extrapolations.
Our tight analysis of \textsc{EG}’s convergence rate in games shows that, unlike in convex minimization, \textsc{EG} may be much faster than \textsc{GD}. 
\end{abstract}

\section{Introduction}
Gradient-based optimization methods have underpinned many of the recent successes of machine learning. The training of many models is indeed formulated as the minimization of a loss involving the data. However, a growing number of frameworks rely on optimization problems that involve multiple players and objectives. For instance, actor-critic models \citep{pfauConnectingGenerativeAdversarial2016}, generative adversarial networks (GANs) \citep{goodfellowGenerativeAdversarialNets2014} and automatic curricula \citep{sukhbaatarIntrinsicMotivationAutomatic2018} can be cast as two-player games.

Hence games are a generalization of the standard single-objective framework. The aim of the optimization is to find \emph{Nash equilibria}, that is to say situations where no player can unilaterally decrease their loss. However, new issues that were not present for single-objective problems arise. The presence of rotational dynamics prevent standard algorithms such as the gradient method to converge on simple bilinear examples \citep{goodfellowNIPS2016Tutorial2016,balduzziMechanicsNPlayerDifferentiable2018b}.
Furthermore, stationary points of the gradient dynamics are not necessarily Nash equilibria \citep{adolphsLocalSaddlePoint2018, mazumdarFindingLocalNash2019a}.

Some recent progress has been made by introducing new methods specifically designed with games or variational inequalities in mind. The main example are the optimistic gradient method (\textsc{OG}) introduced by \citet{rakhlin2013optimization} initially for online learning,  consensus optimization (\textsc{CO}) which adds a regularization term to the optimization problem and the extragradient method (\textsc{EG}) originally introduced by~\citet{g.m.korpelevichExtragradientMethodFinding1976}. Though these news methods and the gradient method (\textsc{GD})
have similar performance in convex optimization, their behaviour seems to differ when applied to games: unlike gradient, they converge on the so-called bilinear example \citep{tsengLinearConvergenceIterative1995,gidelVariationalInequalityPerspective2018a,mokhtariUnifiedAnalysisExtragradient2019,abernethyLastiterateConvergenceRates2019a}. 

However, linear convergence results for \textsc{EG} and \textsc{OG} (a.k.a extrapolation from the past) in particular have only been proven for either strongly monotone variational inequalities problems, which include strongly convex-concave saddle point problems, or in the bilinear setting separately \citep{tsengLinearConvergenceIterative1995,gidelVariationalInequalityPerspective2018a,mokhtariUnifiedAnalysisExtragradient2019}. 
    
In this paper, we study the dynamics of such gradient-based methods and in particular \textsc{GD}, \textsc{EG} and more generally multi-step extrapolations methods for unconstrained games. Our objective is three-fold. First, taking inspiration from the analysis of \textsc{GD} by \citet{gidelNegativeMomentumImproved2018b}, we aim at providing a single precise analysis of \textsc{EG} which covers both the bilinear and the strongly monotone settings and their intermediate cases.
Second, we are interested in theoretically comparing \textsc{EG} to \textsc{GD} and general multi-step extrapolations through upper and lower bounds on convergence rates. Third, we provide a framework to extend the unifying results of spectral analysis in global guarantees and leverage it to prove tighter convergence rates for \textsc{OG} and \textsc{CO}.
Our contributions can be summarized as follows:
    \begin{itemize} %
        \item We perform a spectral analysis of \textsc{EG} in \S\ref{subsection: spectral analysis of eg}. We derive a local rate of convergence which covers the whole range of settings between purely bilinear and strongly monotone games and which is faster than existing rates in some regimes. Our analysis also encompasses multi-step extrapolation methods and highlights the similarity between \textsc{EG} and the proximal point methods. 
        \item We use and extend the framework from \citet{arjevaniLowerUpperBounds2016} to derive lower bounds for specific classes of algorithms.
        \begin{enumerate*}[series = tobecont, itemjoin = \quad, label=(\roman*)]
        \item We show in \S\ref{subsection: lower bound gradient} that the previous spectral analysis of \textsc{GD} by \citet{gidelNegativeMomentumImproved2018b} is tight, confirming the difference of behaviors with \textsc{EG}.
        \item We prove lower bounds for $1$-Stationary Canonical Linear Iterative methods with any number of extrapolation steps in \S\ref{subsection: lower bounds n-extra}. As expected, this shows that increasing this number or choosing different step sizes for each does not yield significant improvements and hence \textsc{EG} can be considered as optimal among this class.
        \end{enumerate*}
        \item In \S\ref{section: global cv rate}, we derive a  global convergence rate for the \textsc{EG} with the same unifying properties as the local analysis. We then leverage our approach to derive global convergence guarantees for \textsc{OG} and \textsc{CO} with similar unifying properies. 
        It shows that, while these methods converges for different reasons in the convex and bilinear settings, in between they actually take advantage of the most favorable one. 
    \end{itemize}
    
\begin{table}[t]
\centering
\resizebox{\columnwidth}{!}{
\begin{tabular}{@{}lP{0.8cm}
P{0.8cm}P{1.1cm}P{1.2cm}P{2cm}@{}}
\toprule
    & \footnotesize \citet{tsengLinearConvergenceIterative1995} & 
     \footnotesize \citet{gidelVariationalInequalityPerspective2018a}&  \footnotesize \citet{mokhtariUnifiedAnalysisExtragradient2019}  &  \footnotesize \citet{abernethyLastiterateConvergenceRates2019a}&
     \small This work \,\; \S\ref{section: global cv rate}
\\ \midrule

\small \textsc{EG}        &       $c\frac{\mu^2}{L^2}$       &       
- & $\frac{\mu}{4L}$    & - & $\frac{1}{4}(\frac{\mu}{L}+\frac{\gamma^2}{16L^2})$ \\[2mm]
\small \textsc{OG}        &  - & 
$\frac{\mu}{4L}$  & $\frac{\mu}{4L}$ & - & $\frac{1}{4}(\frac{\mu}{L}+\frac{\gamma^2}{32L^2})$ \\[2mm]
\small \textsc{CO}        & -  & 
- & - & $\frac{\gamma^2}{4L_H^2}$&  $ \tfrac{\mu^2}{2L_H^2} + \tfrac{\gamma^2}{2L_H^2}$ \\
\bottomrule
\end{tabular}
}
\caption{\small
Summary of the global convergence results presented in~\S\ref{section: global cv rate} for extragradient (\textsc{EG}), optimistic gradient (\textsc{OMD}) and consensus optimization (\textsc{CO}) methods. If a result shows that the iterates converge as $\bigO((1-r)^t)$, the quantity $r$ is reported (the larger the better). The letter $c$ indicates that the numerical constant was not reported by the authors.
$\mu$ is the strong monotonicity of the vector field, $\gamma$ is a global lower bound on the singular values of $\nabla v$ , $L$ is the Lipschitz constant of the vector field and $L_H^2$ the Lipschitz-smoothness of $\frac{1}{2}\|v\|_2^2$. For instance, for the so-called bilinear example (Ex.~\ref{ex: bilinear game}), we have $\mu =0$ and $\gamma = \sigma_{\min}(A)$. Note that for this particular example, previous papers developed a specific analysis that breaks when a small regularization is added (see Ex.~\ref{ex: in between example}). 
}\label{fig: small table}
\varvspace{-4mm}
\end{table}

\section{Related Work}
\textbf{Extragradient} was first introduced by \citet{g.m.korpelevichExtragradientMethodFinding1976} in the context of variational inequalities. \citet{tsengLinearConvergenceIterative1995} proves results which induce linear convergence rates for this method in the bilinear and strongly monotone cases. %
We recover both rates with our analysis. The extragradient method was generalized to arbitrary geometries by \citet{nemirovskiProxMethodRateConvergence2004} as the mirror-prox method. A sublinear rate of $\bigO(1/t)$ was proven for monotone variational inequalities by treating this method as an approximation of the proximal point method as we will discuss later. More recently, \citet{mertikopoulosOptimisticMirrorDescent2018} proved that, for a broad class of saddle-point problems, its stochastic version converges almost surely to a solution. 

\textbf{Optimistic gradient method} is slightly different from \textsc{EG} and can be seen as a kind of extrapolation from the past~\citep{gidelVariationalInequalityPerspective2018a}. It was initially introduced for online learning \citep{chiang2012online,rakhlin2013optimization} and subsequently studied in the context of games by \citet{daskalakisTrainingGANsOptimism2017a}, who proved that this method converges on bilinear games.
\citet{gidelVariationalInequalityPerspective2018a} interpreted GANs as a variational inequality problem and derived \textsc{OG} as a variant of \textsc{EG} which avoids ``wasting" a gradient. They prove a linear convergence rate for strongly monotone variational inequality problems.
Treating \textsc{EG} and \textsc{OG} as perturbations of the proximal point method, \cite{mokhtariUnifiedAnalysisExtragradient2019} gave new but still separate derivations for the standard linear rates in the bilinear and the strongly convex-concave settings. 
\citet{liangInteractionMattersNote2018} mentioned the potential impact of the interaction between the players, but they only formally show this on bilinear examples: our results show that this conclusion extends to general nonlinear games. 

\textbf{Consensus optimization} has been motivated by the use of gradient penalty objectives for the practical training of GANs~\citep{gulrajani2017improved,meschederNumericsGANs2017}. It has been analysed by \citet{abernethyLastiterateConvergenceRates2019a} as a perturbation of Hamiltonian gradient descent. 

We provide a unified and tighter analysis for these three algorithms leading to faster rates (cf.\ Tab.~\ref{fig: small table}).

\textbf{Lower bounds in optimization} date back to \citet{nemirovskyProblemComplexityMethod1983} and were popularized by \citet{nesterovIntroductoryLecturesConvex2004}. 
One issue with these results is that they are either only valid for a finite number of iterations depending on the dimension of the problem or are proven in infinite dimensional spaces. To avoid this issue, \citet{arjevaniLowerUpperBounds2016} introduced a new framework called $p$-Stationary Canonical Linear Iterative algorithms ($p$-SCLI). It encompasses methods which, applied on  quadratics, compute the next iterate as fixed linear transformation of the $p$ last iterates, for some fixed $p \geq 1$. We build on and extend this framework to derive lower bounds for games for 1-SCLI. Concurrenty, \citet{ibrahimLowerBoundsConditioning2019} extended the whole $p$-SCLI framework to games but excluded extrapolation methods. Note that sublinear lower bounds have been proven for saddle-point problems by \citet{nemirovskyInformationbasedComplexityLinear1992a,nemirovskiProxMethodRateConvergence2004,chenAcceleratedSchemesClass2014,ouyangLowerComplexityBounds2018}, but they are outside the scope of this paper since we focus on linear convergence bounds.

Our notation is presented in \S\ref{section: notations}. The proofs can be found in the subsequent appendix sections.

\section{Background and motivation}\label{section: game optimization}

\subsection{\textit{n}-player differentiable games}
Following \citet{balduzziMechanicsNPlayerDifferentiable2018b}, a \emph{$n$-player differentiable game} can be defined as a family of twice continuously differentiable losses $l_i: \R^{d} \rightarrow \R$ for $i = 1,\dots,n$. The parameters for player $i$ are $\omega^i \in \R^{d_i}$ and we note $\omega = (\omega^1,\dots,\omega^n) \in \R^d$ with $d = \sum_{i = 1}^n d_i$.
Ideally, we are interested in finding an \emph{unconstrained Nash equilibrium} \citep{vonneumannTheoryGamesEconomic1944}: that is to say a point $\omega^* \in \R^d$ such that \begin{equation*}
\forall i \in \{1,\dots,n\},\quad (\omega^i)^* \in \argmin_{\omega^i \in \R^{d_i}}  l_i((\omega^{-i})^*,\omega^i)\,,
\varvspace{-2mm}
\end{equation*}
where the vector $(\omega^{-i})^*$ contains all the coordinates of $\omega^*$ except the $i^{th}$ one.
Moreover, we say that a game is \emph{zero-sum} if $\sum_{i = 1}^n l_i = 0$. For instance, following \citet{meschederNumericsGANs2017,gidelNegativeMomentumImproved2018b}, the standard formulation of GANs from \citet{goodfellowGenerativeAdversarialNets2014} can be cast as a two-player zero-sum game.  The Nash equilibrium corresponds to the desired situation where the generator exactly capture the data distribution, completely confusing a perfect discriminator.

Let us now define the \emph{vector field} 
\begin{equation*}
v(\omega) = \begin{pmatrix}
\nabla_{\omega^1} l_1(\omega)\;, & \cdots &,\;
\nabla_{\omega^n} l_n(\omega)
\end{pmatrix}
\end{equation*} 
associated to a $n$-player game
and its Jacobian:
\begin{equation*}
\nabla v(\omega) =
\begin{pmatrix}
\nabla^2_{\omega^1} l_1(\omega) & \dots & \nabla_{\omega^n}\nabla_{\omega^1} l_1(\omega)\\
 \vdots & & \vdots \\ \nabla_{\omega^1}\nabla_{\omega^n} l_n(\omega) &\dots &
 \nabla^2_{\omega^n} l_n(\omega)\\
\end{pmatrix}
\,. \notag
\end{equation*} 
We say that $v$ is \emph{$L$-Lipschitz} for some $L \geq 0$ if $\| v(\omega) - v(\omega')\| \leq L \|\omega - \omega'\|$ $\forall \omega,\omega' \in \R^d$, that $v$ is \emph{$\mu$-strongly monotone} for some $\mu \geq 0$, if   $\mu\|\omega - \omega'\|^2 \leq  (v(\omega) - v(\omega'))^T(\omega - \omega')$ $\forall \omega,\omega' \in \R^d$.

A Nash equilibrium is always a \emph{stationary} point of the gradient dynamics, i.e. a point $\omega \in \R^d$ such that $v(\omega)=0$. However, as shown by \citet{adolphsLocalSaddlePoint2018,mazumdarFindingLocalNash2019a,berard2019closer}, in general, being a Nash equilibrium is neither necessary nor sufficient for being a locally stable stationary point, but if $v$ is monotone, these two notions are equivalent. Hence, in this work we focus on finding stationary points.
One important class of games is \emph{saddle-point problems}: two-player games with $l_1 = - l_2$. If $v$ is monotone, or equivalently $f$ is convex-concave, stationary points correspond to the solutions of the min-max problem
\begin{equation*}
\min_{\omega_1 \in \R^{d_1}}\max_{\omega_2 \in \R^{d_2}} l_1(\omega_1, \omega_2) \,.
\varvspace{-2mm}
\end{equation*}

 \citet{gidelNegativeMomentumImproved2018b} and \citet{balduzziMechanicsNPlayerDifferentiable2018b} mentioned two particular classes of games, which can be seen as the two opposite ends of a spectrum. As the definitions vary, we only give the intuition for these two categories. The first one is \emph{adversarial games}, where the Jacobian has eigenvalues with small real parts and large imaginary parts and the cross terms $\nabla_{\omega_i}\nabla_{\omega_j} l_j(\omega)$, for $i \neq j$, are dominant. Ex.~\ref{ex: bilinear game} gives a prime example of such game that has been heavily studied: a simple bilinear game whose Jacobian is anti-symmetric and so only has imaginary eigenvalues (see Lem.~\ref{lemma: spectrum bilinear game} in App.~\ref{subsection: app spectral analysis of eg}):
 \begin{ex}[Bilinear game]\label{ex: bilinear game}
\varvspace{-1mm}
  \begin{equation*}%
\min_{x \in \R^{m}} \max_{y \in \R^{m}} x^TA y + b^T x + c^Ty 
\varvspace{-1mm}
 \end{equation*}
 with $A \in \R^{m \times m}$ non-singular, $b \in \R^m$ and $c \in \R^m$.
 \end{ex}
 If $A$ is non-singular, there is an unique stationary point which is also the unique Nash equilibrium. The gradient method is known not to converge in such game while the proximal point and extragradient methods converge \cite{rockafellarMonotoneOperatorsProximal1976,tsengLinearConvergenceIterative1995}.
 
Bilinear games are of particular interest to us as they are seen as models of the convergence problems that arise during the training of GANs. Indeed, \citet{meschederNumericsGANs2017} showed that eigenvalues of the Jacobian of the vector field with small real parts and large imaginary parts could be at the origin of these problems. Bilinear games have pure imaginary eigenvalues and so are limiting models of this situation. Moreover, they can also be seen as a very simple type of WGAN, with the generator and the discriminator being both linear, as explained in \citet{gidelVariationalInequalityPerspective2018a,meschederWhichTrainingMethods}.
 
 The other category is \emph{cooperative games}, where the Jacobian has eigenvalues with large positive real parts and small imaginary parts and the diagonal terms $\nabla^2_{\omega_i} l_i$  are dominant. Convex minimization problems are the archetype of such games. Our hypotheses, for both the local and the global analyses, encompass these settings.

\subsection{Methods and convergence analysis}\label{subsection: background}
\textbf{Convergence theory of fixed-point iterations.}
Seeing optimization algorithms as the repeated application of some operator allows us to deduce their convergence properties from the spectrum of this operator. This point of view was presented by \cite{polyakIntroductionOptimization1987a,bertsekasNonlinearProgramming1999} and recently used by \cite{arjevaniLowerUpperBounds2016,meschederNumericsGANs2017,gidelNegativeMomentumImproved2018b}  for instance. 
The idea is that the iterates of a method $(\omega_t)_t$ are generated by a scheme of the form:
\begin{equation*}
    \omega_{t+1} = F(\omega_t)\,,\quad \forall t \geq 0
\end{equation*}
where $F : \R^d \rightarrow \R^d$ is an operator representing the method. Near a stationary point $\omega^*$, the behavior of the iterates is mainly governed by the properties of $\nabla F(\omega^*)$ as
    $F(\omega) - \omega^* \approx \nabla F(\omega^*)(\omega - \omega^*)\,.$
This is formalized by the following classical result:
\begin{thm}[\citet{polyakIntroductionOptimization1987a}]\label{thm: local convergence}
Let $F: \R^d \longrightarrow \R^d$ be continuously differentiable and let $\omega^* \in \R^d$ be a fixed point of $F$. If $\rho(\nabla F(\omega^*)) < 1$, then for $\omega_0$ in a neighborhood of $\omega^*$, the iterates $(\omega_t)_t$ defined by $\omega_{t+1} = F(\omega_t)$ for all $t \geq 0$ converge linearly to $\omega^*$ at a rate of $\bigO((\rho(\nabla F(\omega^*))+\epsilon)^t)$ for all $\epsilon > 0$.
\end{thm}
This theorem means that to derive a local rate of convergence for a given method, one needs only to focus on the eigenvalues of $\nabla F(\omega^*)$.
Note that if the operator $F$ is linear, there exists slightly stronger results such as \Cref{thm: global convergence for constant jacobian} in \cref{section: complement background}. %

\textbf{Gradient method.} 
Following \cite{gidelNegativeMomentumImproved2018b}, we define \textsc{GD} as the application of the operator $F_\eta (\omega) \coloneqq \omega - \eta v(\omega)$, for $\omega \in \R^d$. Thus we have:
\begin{equation}
    \omega_{t+1} = F_\eta(\omega_t) = \omega_t - \eta v(\omega_{t})\,. \tag{GD}
    \varvspace{-2mm}
\end{equation}

\textbf{Proximal point.} 
For $v$ monotone \citep{mintyMonotoneNonlinearOperators1962, rockafellarMonotoneOperatorsProximal1976}, the proximal point operator can be defined as $P_\eta(\omega) = (\Id + \eta v)^{-1}(\omega)$
and therefore can be seen as an implicit scheme:
$\omega_{t+1} = \omega_t - \eta v(\omega_{t+1})$.

\textbf{Extragradient.}
\textsc{EG} was introduced by \citet{g.m.korpelevichExtragradientMethodFinding1976} in the context of variational inequalities. Its update rule is
\begin{equation}\omega_{t+1} = \omega_t - \eta v(\omega_{t} - \eta v(\omega_t))\,.
\label{eq: extrapolation method}
\tag{EG}\end{equation}
It can be seen as an approximation of the implicit update of the proximal point method. Indeed \citet{nemirovskiProxMethodRateConvergence2004} showed a rate of $\bigO(1/t)$ for extragradient by treating it as a ``good enough" approximation of the proximal point method. To see this, fix $\omega \in \R^d$. Then $P_\eta(\omega)$ is the solution of $z = \omega - \eta v(z)$. Equivalently, $P_\eta(\omega)$ is the fixed point of 
\begin{equation}
\varphi_{\eta,\omega}: z \longmapsto \omega - \eta v(z)\,,
\end{equation}
which is a contraction for $\eta > 0$ small enough. From Picard's fixed point theorem, one gets that the proximal point operator $P_\eta(\omega)$ can be obtained as the limit of $\varphi_{\eta, \omega}^k(\omega)$ when $k$ goes to infinity. What \citet{nemirovskiProxMethodRateConvergence2004} showed is that $\varphi_{\eta, \omega}^2(\omega)$, that is to say the extragradient update, is close enough to the result of the fixed point computation to be used in place of the proximal point update without affecting the sublinear convergence speed. Our analysis of multi-step extrapolation methods will encompass all the iterates $\varphi_{\eta, \omega}^k$ and we will show that a similar phenomenon happens for linear convergence rates. 

\textbf{Optimistic gradient.} Originally introduced in the online learning literature~\citep{chiang2012online,rakhlin2013optimization} as a two-steps method, \citet{daskalakisTrainingGANsOptimism2017a} reformulated it with only one step in the unconstrained case:
\begin{equation}\label{equation: optimistic method}
    w_{t+1} = w_t - 2\eta v(w_t) + \eta v(w_{t-1})    \,.
    \tag{OG}
    \varvspace{-2mm}
\end{equation}

\textbf{Consensus optimization.} Introduced by 
\citet{meschederNumericsGANs2017} in the context of games, consensus optimization is a second-order yet efficient method, as it only uses a Hessian-vector multiplication whose cost is the same as two gradient evaluations \citep{pearlmutterFastExactMultiplication1994}.
We define the CO update as:
\begin{equation}\label{eq: co iterates}
\omega_{t+1} = \omega_t - (\alpha v(\omega_t) + \beta \nabla H(\omega_t))\, \tag{CO}
\end{equation}
where $H(\omega) = \frac{1}{2}\|v(\omega)\|^2_2$ and $\alpha, \beta > 0$ are step sizes.

\subsection{\textit{p}-SCLI framework for game optimization}\label{subsection: pSCLI framework}
In this section, we present an extension of the framework of \citet{arjevaniLowerUpperBounds2016} to derive lower bounds for game optimization (also see \S\ref{section: p-scli game optimization}). 
The idea of this framework is to see algorithms as the iterated application of an operator. 
If the vector field is linear, this transformation is linear too and so its behavior when iterated is mainly governed by its spectral radius. 
This way, showing a lower bound for a class of algorithms is reduced to lower bounding a class of spectral radii.

We consider $\mathcal{V}_d$ the set of linear vector fields $v:\R^d \rightarrow \R^d$, i.e., vector fields $v$ whose Jacobian $\nabla v$ is a constant  $d\times d$ matrix.\footnote{With a slight abuse of notation, we also denote by $\nabla v$ this matrix.}
The class of algorithms we consider is the class of \emph{$1$-Stationary Canonical Linear Iterative algorithms  ($1$-SCLI)}. Such an algorithm is defined by a mapping $\mathcal{N}: \R^{d\times d} \rightarrow \R^{d\times d}$.
The associated update rule can be defined through,
\begin{equation}\label{eq: 1-SCLI rule}
F_{\mathcal \N}(\omega) = w + \N(\nabla v)v(\omega)\,\quad \forall \omega \in \R^d\,,
\end{equation}
This form of the update rule is required by the consistency condition of \citet{arjevaniLowerUpperBounds2016} which is necessary for the algorithm to converge to stationary points, as discussed in \S\ref{section: p-scli game optimization}. Also note that $1$-SCLI are first-order methods that use only the last iterate to compute the next one. Accelerated methods such as accelerated gradient descent \citep{nesterovIntroductoryLecturesConvex2004} or the heavy ball method \citep{polyakMethodsSpeedingConvergence1964} belong in fact to the class of $2$-SCLI, which encompass methods which uses the last two iterates.

As announced above, the spectral radius of the operator gives a lower bound on the speed of convergence of the iterates of the method on affine vector fields, which is sufficient to include bilinear games, quadratics and so strongly monotone settings too.
\begin{thm}[name=\citet{arjevaniLowerUpperBounds2016}, restate=thmLowerBoundFromSpectralRadius]\label{thm: lower bound from spectral radius}
For all $v \in \mathcal{V}_d$, for almost every\footnote{For any measure absolutely continuous w.r.t.~the Lebesgue measure.} initialization point $\omega_0 \in \R^d$, if $(\omega_t)_t$ are the iterates of $F_{\N}$ starting from $\omega_0$, 
\begin{equation*}
\|\omega_t - \omega^*\| \geq \Omega(\rho(\nabla F_{\N})^t\|\omega_0 - \omega^*\|).
\varvspace{-2mm}
\end{equation*}
\end{thm}
\section{Revisiting \textsc{GD} for games}\label{section: revisiting}

\label{subsection: lower bound gradient}

In this section, our goal is to illustrate the precision of the spectral bounds and the complexity of the interactions between players in games. 
We first give a simplified version of the bound on the spectral radius from \citet{gidelNegativeMomentumImproved2018b} and show that their results also imply that this rate is tight. 
\begin{thm}[{restate=[name=]thmSpectralRateGradient}]\label{thm: spectral rate for gradient}
Let $\omega^*$ be a stationary point of $v$ and denote by $\sigma^*$ the spectrum of $\nabla v(\omega^*)$.
If the eigenvalues of $\nabla v(\omega^*)$ all have positive real parts, then
\varvspace{-2mm}
\begin{enumerate}[label=(\roman*)., font=\itshape]
    \item \citep{gidelNegativeMomentumImproved2018b} For $\eta=\min_{\lambda \in \sigma^*} \Re(1/\lambda)$, the spectral radius of $F_\eta$ can be upper-bounded as
    \begin{equation*}
\rho(\nabla F_\eta(\omega^*))^2 \leq 1 -  \min_{\lambda \in \sigma^*} \Re(1/\lambda) \min_{\lambda \in \sigma^*} \Re(\lambda)\,.
\end{equation*}
\varvspace{-8mm}
\item For all $\eta > 0$, the spectral radius of the gradient operator $F_\eta$ at  $\omega^*$ is lower bounded by
\begin{equation*}
\rho(\nabla F_\eta(\omega^*))^2 \geq 1 -  4\min_{\lambda \in \sigma^*} \Re(1/\lambda) \min_{\lambda \in \sigma^*} \Re(\lambda)\,.
\end{equation*}
\varvspace{-9mm}
\end{enumerate}
\end{thm}
This result is stronger than what we need for a standard lower bound: using \Cref{thm: lower bound from spectral radius}, this yields a lower bound on the convergence of the iterates for all games with affine vector fields.

\label{subsection: comparison}
 We then consider a saddle-point problem, and under some assumptions presented below, one can interpret the spectral rate of the gradient method mentioned earlier in terms of  the standard strong convexity and Lipschitz-smoothness constants. There are several cases, but one of them is of special interest to us as it demonstrates the precision of spectral bounds. %
\begin{ex}[Highly adversarial saddle-point problem]\label{ex: highly adversarial game}%
Consider $\min_{x \in \R^m}\max_{y \in \R^m} f(x, y)$ with $f$ twice differentiable such that
\begin{enumerate}[label=(\roman*)., font=\itshape]
    \item $f$ satisfies, with $\mu_1,\mu_2$ and $\mu_{12}$ non-negative,
    \begin{align*}
&\mu_1 I \preccurlyeq  \nabla_x^2 f \preccurlyeq L_1 I \,,
\quad  \mu_2 I  \preccurlyeq  -\nabla_y^2 f \preccurlyeq L_2 I \\ 
&\mu_{12}^2 I \preccurlyeq  (\nabla_x \nabla_y f)^T(\nabla_x \nabla_y f) \preccurlyeq L_{12}^2 I\,,
\end{align*}
such that $\mu_{12} > 2\max(L_1-\mu_2, L_2-\mu_1)$.
\item There exists a stationary point $\omega^* = (x^*, y^*)$ and at this point, $\nabla_y^2 f(\omega^*)$ and $\nabla_x \nabla_y f (\omega^*)$ commute
and $\nabla_x^2 f(\omega^*)$, $\nabla_y^2 f(\omega^*)$ and $(\nabla_x \nabla_y f (\omega^*))^T(\nabla_x \nabla_y f (\omega^*))$ commute.
\end{enumerate}
\end{ex}

Assumption \textit{(i)} corresponds to a highly adversarial setting as the coupling (represented by the cross derivatives) is much bigger than the Hessians of each player. 
Assumption \textit{(ii)} is a technical assumption needed to compute a precise bound on the spectral radius and holds if, for instance, the objective is separable, i.e.
$f(x, y) = \sum_{i = 1}^m f_i(x_i,y_i)$.
Using these assumptions,  we can upper bound the rate of \Cref{thm: spectral rate for gradient} as follows:
\begin{cor}[{restate=[name=]corInterpretedRateOfGradient}]\label{cor: interpreted rate of gradient}
Under the assumptions of \Cref{thm: spectral rate for gradient} and Ex.~\ref{ex: highly adversarial game},
\varvspace{-2mm}
\begin{equation}\label{eq: interpreted rate of gradient}
\rho(\nabla F_\eta(\omega^*))^2 \leq 1 - \tfrac{1}{4}\tfrac{(\mu_1 + \mu_2)^2}{L_{12}^2 + L_1L_2} \,.
\varvspace{-2mm}
\end{equation}
\end{cor}
What is surprising is that, in some regimes, this result induces faster local convergence rates than the existing upper-bound for \textsc{EG}~\citep{tsengLinearConvergenceIterative1995}:
\begin{equation}\label{eq: rate eg litterature}
1 - \tfrac{\min(\mu_1,\mu_2)}{4L_{max}}\;\;
\text{where} \;\; L_{max} = \max(L_1, L_2, L_{12}) \,. 
\end{equation} 
If, say,  $\mu_2$ goes to zero, that is to say the game becomes unbalanced, the rate of \textsc{EG} goes to~1 while the one of \cref{eq: interpreted rate of gradient} stays bounded by a constant which is strictly less than 1. Indeed, the rate of \Cref{cor: interpreted rate of gradient} involves the arithmetic mean of $\mu_1$ and $\mu_2$, which is roughly the maximum of them, while \cref{eq: rate eg litterature} makes only the minimum of the two appear. This adaptivity to the best strong convexity constant is not present in the standard convergence rates of the \textsc{EG} method.
We remedy this situation with a new analysis of \textsc{EG} in the following section.
\section{Spectral analysis of multi-step \textsc{EG}}\label{section: convergence analysis}
\label{subsection: spectral analysis of eg}
In this section, we study the local dynamics of \textsc{EG} and, more generally, of extrapolation methods. Define a \emph{$k$-extrapolation method} ($k$-\textsc{EG)} by the operator
\begin{equation}\label{eq: k-extrapolation methods}
F_{k,\eta}: \omega \mapsto \varphi^k_{\eta,\omega}(\omega)\;\;\text{with}\;\; \varphi_{\eta,\omega}: z \mapsto \omega - \eta v(z)\,.
\end{equation}

We are essentially considering all the iterates of the fixed point computation discussed in \cref{subsection: background}. Note that $F_{1,\eta}$ is \textsc{GD} while $F_{2,\eta}$ is \textsc{EG}. %
We aim at studying the local behavior of these methods at stationary points of the gradient dynamics, so fix $\omega^*$ s.t. $v(\omega^*) = 0$ and let $\sigma^* = \Sp \nabla v(\omega^*)$. We compute the spectra of these operators at this point and this immediately yields the spectral radius on the proximal point operator:
\begin{lemma}[{restate=[name=]lemmaSpectraOperators}]\label{lemma: spectra operators} 
The spectra of the $k$-extrapolation operator and the proximal point operator are given by:
\begin{align*}
&\Sp \nabla F_{\eta, k}(\omega^*) = \big\{ \textstyle\sum_{j = 0}^k (-\eta \lambda)^j\ |\ \lambda \in \sigma^*\big\} \\
& \text{and}\ \ 
\Sp \nabla P_\eta(\omega^*) = \left\{ (1 + \eta \lambda)^{-1}\ |\ \lambda \in \sigma^*\right\}\,.
\end{align*}
Hence, for all $\eta > 0$, the spectral radius of the operator of the proximal point method is equal to: \begin{equation}\label{eq: rate proximal}
\rho(\nabla P_{\eta}(\omega^*))^2 = 1 - \min_{\lambda \in \sigma^*} \tfrac{2\eta \Re \lambda + \eta^2|\lambda|^2}{|1+\eta \lambda|^2}\,.
\varvspace{-2mm}
\end{equation}
\end{lemma}
Again, this shows that a $k$-\textsc{EG} is essentially an approximation of proximal point for small step sizes as $(1 + \eta \lambda)^{-1} = \sum_{j = 0}^k (-\eta \lambda)^j + \bigO\left(|\eta \lambda|^{k+1}\right)$. This could suggest that increasing the number of extrapolations might yield better methods but we will actually see that $k=2$ is enough to achieve a similar rate to proximal. %
We then bound the spectral radius of $\nabla F_{\eta, k}(\omega^*)$:
\begin{thm}[{restate=[name=]thmSpectralRateExtragradient}]\label{thm: rate of n-extrapolation}
Let $\sigma^* = \Sp \nabla v(\omega^*)$.
If the eigenvalues of $\nabla v(\omega^*)$ all have non-negative real parts, the spectral radius of the $k$-extrapolation method for $k \geq 2$ satisfies:
\begin{equation}\label{eq: first rate eg}
\rho(\nabla F_{\eta,k}(\omega^*))^2 \leq 1 - \min_{\lambda \in \sigma^*} \tfrac{2\eta \Re \lambda + \tfrac{7}{16}\eta^2|\lambda|^2}{|1+\eta \lambda|^2}\,,
\varvspace{-2mm}
\end{equation}
$\forall \eta \leq \frac{1}{4^{\frac{1}{k-1}}}\frac{1}{\max_{\lambda \in \sigma^*}|\lambda|}$.
For $\eta = (4\max_{\lambda \in \sigma^*}|\lambda|)^{-1}$, this can be simplified as (noting $\rho := \rho(\nabla F_{\eta,k}(\omega^*))$):
\begin{equation}\label{eq: second rate eg}
\rho^2 \leq 1 - \tfrac{1}{4}\left(\tfrac{\min_{\lambda \in \sigma^*} \Re \lambda}{\max_{\lambda \in \sigma^*} |\lambda|}+\tfrac{1}{16}\tfrac{\min_{\lambda \in \sigma^*}|\lambda|^2}{\max_{\lambda \in \sigma^*}|\lambda|^2}\right)\,.
\varvspace{-2mm}
\end{equation}
\end{thm}
The zone of convergence of extragradient as provided by this theorem is illustrated in Fig.~\ref{fig: eg zone convergence}.

The bound of \cref{eq: second rate eg} involves two terms: the first term can be seen as the strong monotonicity of the problem, which is predominant in convex minimization problems, while the second shows that even in the absence of it, this method still converges, such as in bilinear games.  Furthermore, in situation in between, this bound shows that the extragradient method exploits the biggest of these quantities as they appear as a sum as illustrated by the following simple example.
\begin{ex}[``In between" example]\label{ex: in between example}
\begin{equation*}
\min_{x \in \R} \max_{y \in \R} \tfrac{\epsilon}{2} (x^2 - y^2) + xy \,,\quad \text{for }1 \geq \epsilon > 0
\varvspace{-2mm}
\end{equation*}
\end{ex}
Though for $\epsilon$ close to zero, the dynamics will behave as such, this is not a purely bilinear game.  The associated vector field is only $\epsilon$-strongly monotone and convergence guarantees relying only on strong monotonicity would give a rate of roughly $1 - \epsilon/4$. However \cref{thm: rate of n-extrapolation} yields a convergence rate of roughly $1 - 1/64$ for extragradient.

\newcommand{\spectrumcolor}{magenta}
\begin{SCfigure}
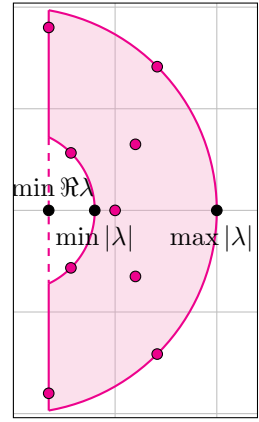

\hspace{-4mm}
\drawSpectrum[.48\textwidth]
\hspace{-2mm}
\caption{\small 
Illustration of the three quantities involved in \cref{thm: rate of n-extrapolation}. The magenta dots are an example of eigenvalues belonging to $\sigma^*$. Note that $\sigma^*$ is always symmetric with respect to the real axis because the Jacobian is a real matrix (and thus non-real eigenvalues are complex conjugates).
Note how $\min \Re\lambda$ may be significantly smaller that $\min |\lambda|$.
\varvspace{-3mm}
}\label{fig: eg zone convergence}
\varvspace{-3mm}
\end{SCfigure}

\textbf{Similarity to the proximal point method.}
First, note that the bound \cref{eq: first rate eg} is surprisingly close to the one of the proximal method \cref{eq: rate proximal}. 
However, one can wonder why the proximal point converges with any step size --- and so arbitrarily fast --- while it is not the case for the $k$-EG, even as $k$ goes to infinity. The reason for this difference is that for the fixed point iterates to converge to the proximal point operator, one needs $\varphi_{\eta, \omega}$ to be a contraction and so to have $\eta$ small enough, at least $\eta < (\max_{\lambda \in \sigma^*}|\lambda|)^{-1}$ for local guarantees. This explains the bound on the step size for 
\ifcompress
$k$-\textsc{EG} .
\fi

\textbf{Comparison with the gradient method.}
We can now compare this result for \textsc{EG} with the convergence rate of the gradient method \cref{thm: spectral rate for gradient} which was shown to be tight.
In general $\min_{\lambda \in \sigma^*} \Re(1/\lambda) \leq (\max_{\lambda \in \sigma^*} |\lambda|)^{-1}$ and, for adversarial games, the first term can be arbitrarily smaller than the second one. %
Hence, in this setting which is of special interest to us, \textsc{EG} has a much faster convergence speed than \textsc{GD}. %

\textbf{Recovery of known rates.}
If $v$ is $\mu$-strongly monotone and $L$-Lipschitz, this bound is at least as precise as the standard one $1 - \mu/(4L)$ as $\mu$ lower bounds the real part of the eigenvalues of the Jacobian, and $L$ upper bounds their magnitude, as shown in \Cref{lemma: relation global quantities to spectrum} in \S\ref{app:proof}. We empirically evaluate the improvement over this standard rate on synthetic examples in \cref{section: app improvement of rate}. On the other hand, \Cref{thm: rate of n-extrapolation} also recovers the standard rates for the bilinear problem,\footnote{Note that by exploiting the special structure of the bilinear game and the fact that $k=2$, one could derive a better constant in the rate. Moreover, our current spectral tools cannot handle the singularity which arises if the two players have a different number of parameters. We provide sharper results to handle this difficulty in \cref{section: app handling singularity}.} as shown below:
\begin{cor}[name=Bilinear game, restate=corRateEGBilinearGame]\label{cor: rate of eg for bilinear game}
Consider Ex.~\ref{ex: bilinear game}. The iterates of the $k$-extrapolation method with $k\geq 2$ converge globally to $\omega^*$ at a linear rate of $\bigO\big(\big(1 - \frac{1}{64}\frac{\sigma_{min}(A)^2}{{\sigma_{max}(A)^2}}\big)^t\big)$.
\end{cor}

Note that this rate is similar %
to the one derived by \citet{gidelNegativeMomentumImproved2018b} for  alternating gradient descent with negative momentum. 
This raises the question of whether general acceleration exists for games, as we would expect the quantity playing the role of the condition number in \Cref{cor: rate of eg for bilinear game} to appear without the square in the convergence rate of a method using momentum.

Finally it is also worth mentioning that the bound of \Cref{thm: rate of n-extrapolation} also displays the adaptivity discussed in \S\ref{subsection: comparison}. Hence, the bound of \Cref{thm: rate of n-extrapolation} can be arbitrarily better than the rate \cref{eq: rate eg litterature} for \textsc{EG} from the literature and also better than the global convergence rate we prove below.%

\textbf{Lower bounds for extrapolation
methods.}\label{subsection: lower bounds n-extra}
We now show that the rates we proved for \textsc{EG} are tight and optimal by deriving lower bounds of convergence for general extrapolation methods.
As described in \cref{subsection: pSCLI framework}, a 1-SCLI method is parametrized by a polynomial $\N$. We consider the class of methods where $\N$ is any polynomial of degree at most $k - 1$, and we will derive lower bounds for this class. This class is large enough to include all the $k'$-extrapolation methods for $k' \leq k$  with possibly different step sizes for each extrapolation step (see \S\ref{subsection: app lower bounds n-extra} for more examples).

Our main result is that no method of this class can significantly beat the convergence speed of \textsc{EG} of \Cref{thm: rate of n-extrapolation} and \Cref{thm: global rate eg}. 
We proceed in two steps: for each of the two terms of these bounds, we provide an example matching it up to a factor. In $(i)$ of the following theorem, we give an example of convex optimization problem which matches the real part, or strong monotonicity, term. Note that this example is already an extension of \citet{arjevaniLowerUpperBounds2016} as the authors only considered constant $\N$. Next, in $(ii)$, we match the other term with a bilinear game example. %
\begin{thm}[{restate=[name=]thmLowerBoundConvex}]\label{thm: lower bound convex}
Let $0 < \mu, \gamma < L$. 
$(i)$ If $d -2 \geq k \geq 3 $, there exists $v \in \mathcal{V}_d$ with a symmetric positive Jacobian whose spectrum is in $[\mu, L]$, such that for any $\N$real polynomial of degree at most $k - 1$, 
\ifcompress
$\rho(F_{\N}) \geq 1 -  \frac{4k^3}{\pi}\frac{\mu}{L}\,.$
\fi

$(ii)$ If ${d}/{2} - 2 \geq k/2 \geq 3$ and $d$ is even, there exists $v \in \mathcal{V}_d$ $L$-Lipschitz with $\min_{\lambda \in \Sp \nabla v} |\lambda| = \sigma_{min}(\nabla v) \geq \gamma$ corresponding to a bilinear game of \Cref{ex: bilinear game} with $ m = d/2$, such  that, for any $\N$real polynomial of degree at most $k - 1$,
\ifcompress
$\rho(F_{\N}) \geq 1 - \frac{k^3}{2\pi} \frac{\gamma^2}{L^2}\,.$
\fi
\end{thm}

First, these lower bounds show that both our convergence analyses of \textsc{EG} are tight, by looking at them for $k=3$ for instance.
Then, though these bounds become looser as $k$ grows, they still show that the potential improvements are not significant in terms of conditioning,  especially compared to the change of regime between \textsc{GD}  and \textsc{EG} . Hence, they still essentially match the convergence speed of \textsc{EG} of \Cref{thm: rate of n-extrapolation} or \Cref{thm: global rate eg}.  Therefore, \textsc{EG} can be considered as optimal among the general class of algorithms which uses at most  a fixed number of composed gradient evaluations and only the last iterate. In particular, there is no need to consider algorithms with more extrapolation steps or with different step sizes for each of them as it only yields a constant factor improvement.

\section{Unified global proofs of convergence}\label{section: global cv rate}
We have shown in the previous section that a spectral analysis of \textsc{EG} yields tight and unified convergence guarantees. We now demonstrate how, combining the strong monotonicity assumption and Tseng's error bound, global convergence guarantees with the same unifying properties might be achieved. 

\subsection{Global Assumptions}

\citet{tsengLinearConvergenceIterative1995} proved linear convergence results for \textsc{EG} by using the projection-type error bound \citet[Eq.~5]{tsengLinearConvergenceIterative1995} which, in the unconstrained case, i.e. for $v(\omega^*) = 0$, can be written as,
\begin{equation}\label{eq: condition tseng}
\gamma \|\omega- \omega^*\|_2 \leq \|v(\omega)\|_2\quad \forall \omega \in \R^d.
\end{equation}
The author then shows that this condition holds for the bilinear game of \cref{ex: bilinear game} and that it induces a convergence rate of $1 - c\sigma_{min}(A)^2/\sigma_{max}(A)^2$ for some constant $c > 0$. He also shows that this condition is implied by strong monotonicity with $\gamma = \mu$.
Our analysis builds on the results from \citet{tsengLinearConvergenceIterative1995} and extends them to cover the whole range of games and recover the optimal rates.%

To be able to interpret Tseng's error bound \cref{eq: condition tseng}, as a property of the Jacobian $\nabla v$, we slightly relax it to,
\begin{equation}\label{eq: weak strong monotonicity}
\gamma \|\omega - \omega'\|_2 \leq \|v(\omega) - v(\omega')\|_2,\quad \forall \omega,\omega' \in \R^d\,.
\end{equation}
This condition can indeed be related to the properties of $\nabla v$ as follows:
\begin{lemma}[{restate=[name=]lemmaWeakStrongMonotonicity}]\label{lemma: weak strong monotonicity}
Let $v$ be continuously differentiable and $\gamma > 0$ : \cref{eq: weak strong monotonicity} holds if and only if $\sigma_{min}(\nabla v) \geq \gamma$.
\end{lemma}
Hence, $\gamma$ corresponds to a lower bound on the singular values of $\nabla v$. This can be seen as a weaker ``strong monotonicity" as it is implied by strong monotonicity, with $\gamma = \mu$, but it also holds for a square non-singular bilinear example of \Cref{ex: bilinear game} with $\gamma = \sigma_{min}(A)$. 

As announced, we will combine this assumption with the strong monotonicity to derive unified global convergence guarantees. Before that, note that this quantities can be related to the spectrum of $\Sp \nabla v(\omega^*)$ as follows -- see \cref{lemma: relation global quantities to spectrum} in \cref{subsection: alternative charact},
\begin{equation}
        \mu \leq \Re(\lambda),\quad \gamma \leq |\lambda| \leq L\,, \quad \forall \lambda \in \Sp \nabla v(\omega^*)\,.
\end{equation}
Hence, theses global quantities are less precise than the spectral ones used in \cref{thm: rate of n-extrapolation}, so the following global results will be less precise than the previous ones.
\subsection{Global analysis \textsc{EG} and \textsc{OG}}
We can now state our global convergence result for \textsc{EG}:
\varvspace{-4mm}
\begin{thm}[{restate=[name=]thmglobalconvergence}]
\label{thm: global rate eg}
Let $v : \R^d \rightarrow \R^d$ be continuously differentiable and \begin{enumerate*}[series = tobecont, itemjoin = \quad, label=(\roman*)] \item $\mu$-strongly monotone for some $\mu \geq 0$, \item $L$-Lipschitz, \item such that $\sigma_{min}(\nabla v) \geq \gamma$ for some $\gamma > 0$.\end{enumerate*} Then, for $\eta \leq (4L)^{-1}$, the iterates $(\omega_t)_t$ of~\eqref{eq: extrapolation method}  converge linearly to $\omega^*$ as, for all $t \geq 0$,
\varvspace{-1mm}
\begin{equation*}
\|\omega_t - \omega^*\|_2^2 \leq \left( 1 - \eta\mu - \tfrac{7}{16}\eta^2\gamma^2\right)^t \|\omega_0 - \omega^*\|_2^2\,.
\varvspace{-2mm}
\end{equation*}
\end{thm}
As for \Cref{thm: rate of n-extrapolation}, this result not only recovers both the bilinear and the strongly monotone case, but shows that \textsc{EG} actually gets the best of both world when in between. %
Furthermore this rate is surprisingly similar to the result of \Cref{thm: rate of n-extrapolation} though less precise, as discussed.

Combining our new proof technique and the analysis provided by~\citet{gidelVariationalInequalityPerspective2018a}, we can derive a similar convergence rate for the optimistic gradient method. 
\begin{thm}[{restate=[name=]thmglobalconvergenceoptimistic}]
\label{thm: global rate og}
Under the same assumptions as in Thm.~\ref{thm: global rate eg}, for $\eta \leq (4L)^{-1}$, the iterates $(\omega_t)_t$ of~\eqref{equation: optimistic method} converge linearly to $\omega^*$ as, for all $t\geq 0$,
\varvspace{-1mm}
\begin{equation*}
\|\omega_t - \omega^*\|_2^2 \leq 2\left( 1 - \eta\mu - \tfrac{1}{8}\eta^2\gamma^2\right)^{t+1} \|\omega_0 - \omega^*\|_2^2\,. 
\varvspace{-2mm}
\end{equation*}
\end{thm}

\textbf{Interpretation of the condition numbers.}
As in the previous section, this rate of convergence for \textsc{EG} is similar to the rate of the proximal point method for a small enough step size, as shown by \Cref{prop: global rate proximal} in \S\ref{app:proof}. Moreover, the proof of the latter gives insight into the two quantities appearing in the rate of \Cref{thm: global rate eg}. Indeed, the convergence result for the proximal point method is obtained by bounding the singular values of $\nabla P_\eta$, and so we compute,\footnote{We dropped the dependence on $\omega$ for compactness.}
\begin{equation*}%
(\nabla P_\eta)^T\nabla P_\eta = \left(I_d + \eta\mathcal{H}(\nabla v) + \eta^2\nabla v\nabla v^T\right)^{-1}
\end{equation*}
where $\mathcal{H}(\nabla v) := \frac{\nabla v + \nabla v^T}{2}\,.$ This explains the quantities ${L}/{\mu}$ and ${L^2}/{\gamma^2}$ appear in the convergence rate, as the first corresponds to the condition number of $\mathcal{H}(\nabla v)$ and the second to the condition number of $\nabla v\nabla v^T$. Thus, the proximal point method uses information from both matrices to converge, and so does \textsc{EG}, explaining why it takes advantage of the best conditioning. 
\subsection{Global analysis of consensus optimization}\label{subsection: consensus}
In this section, we give a unified proof of \textsc{CO}. A global convergence rate for this method was proven by \citet{abernethyLastiterateConvergenceRates2019a}. However it used a perturbation analysis of \textsc{HGD}. The drawbacks are that it required that the CO update be sufficiently close to the one of \textsc{HGD} and could not take advantage of strong monotonicity.
\ifcompress
Here, we combine the monotonicity $\mu$ with the lower bound on the singular value $\gamma$.
\fi

As this scheme uses second-order\footnote{W.r.t.~the losses.}
information, we need to replace the Lipschitz hypothesis with one that also controls the variations of the Jacobian of $v$: we use $L_H^2$, the Lispchitz smoothness of $H$.
See \citet{abernethyLastiterateConvergenceRates2019a} for how it might be instantiated.

\begin{thm}\label{thm: consensus}
Let $v : \R^d \rightarrow \R^d$ be continuously differentiable such that 
\begin{enumerate*}[series = tobecont, itemjoin = \quad, label=(\roman*)]
\item $v$ is $\mu$- strongly monotone for some $\mu \geq 0$,
\item $\sigma_{min}(\nabla v) \geq \gamma$ for some $\gamma > 0$
\item  $H$ is $L_H^2$ Lipschitz-smooth.
\end{enumerate*} Then, for
$\alpha = (\mu + \sqrt{\mu^2 + 2\gamma^2})/(4 L_H^2)$, $\beta = (2L_H^2)^{-1}$ the iterates of CO defined by \cref{eq: co iterates} satisfy, for all $t\geq 0$,
\varvspace{-1mm}
\begin{equation}
H(\omega_t) \leq \left(1 - \tfrac{\mu^2}{2L_H^2} - \left(1 + \tfrac{\mu}{\gamma}\right)\tfrac{\gamma^2}{2L_H^2}\right)^t H(\omega_0)\,.\notag
\varvspace{-1mm}
\end{equation}
\end{thm}
This result shows that CO has the same unifying properties as \textsc{EG}, though the dependence on $\mu$ is worse.

This result also encompasses the rate of \textsc{HGD}  \citep[Lem.~4.7]{abernethyLastiterateConvergenceRates2019a}. The dependance in $\mu$ is on par with the standard rate for the gradient method (see \citet[Eq.~2.12]{nesterovSolvingStronglyMonotone2006} for instance). However, this can be improved using a sharper assumption, as discussed in \cref{rmk: co-coercivity} in \cref{subsection: app proof consensus}, and so our result is not optimal in this regard.

\section{Conclusion}
In this paper, we studied the dynamics of \textsc{EG}, both locally and globally and extended our global guarantees to other promising methods such as \textsc{OG} and \textsc{CO}. Our analysis is tight for \textsc{EG} and unified as they cover the whole spectrum of games from bilinear to purely cooperative settings. They show that in between, these methods enjoy the best of both worlds.  We confirm that, unlike in convex minimization, the behaviors of \textsc{EG} and \textsc{GD} differ significantly. The other lower bounds show that \textsc{EG} can be considered as optimal among first-order methods that use only the last iterate. %

Finally, 
as mentioned in \S\ref{subsection: spectral analysis of eg}, the rate of alternating gradient descent with negative momentum from \cite{gidelNegativeMomentumImproved2018b} on the bilinear example essentially matches the rate of \textsc{EG} in \Cref{cor: rate of eg for bilinear game}. Thus the question of an acceleration for adversarial games similar to the one in the convex case using Polyak~\citep{polyakMethodsSpeedingConvergence1964} or Nesterov's~\citep{nesterovIntroductoryLecturesConvex2004} momentum remains open.

\subsubsection*{Acknowledgments} %
This research was partially supported by the Canada CIFAR AI Chair Program, the Canada Excellence Research Chair in ``Data Science for Realtime Decision-making", by the NSERC Discovery Grants RGPIN-2017-06936 and RGPIN-2019-06512, the FRQNT new researcher program 2019-NC-257943, by a Borealis AI fellowship, a Google Focused Research award and a startup grant by IVADO. Simon Lacoste-Julien is a CIFAR Associate Fellow in the Learning in Machines \& Brains program. The authors would like to thank Adam Ibrahim for helpful discussions, especially on lower bounds.

\bibliographystyle{plainnat}
\bibliography{references}
\newpage
\appendix
\onecolumn
\section{Notation}\label{section: notations}
We denote by $\Sp(A)$ the spectrum of a matrix $A$. Its spectral radius is defined by $\rho(A) = \max \{|\lambda|\ |\ \lambda \in \Sp(A)\}$. We write $\sigma_{min}(A)$ for the smallest singular value of $A$, and $\sigma_{max}(A)$ for the largest. $\Re$ and $\Im$ denote respectively the real part and the imaginary part of a complex number. We write $A \preccurlyeq B$ for two symmetric real matrices if and only if $B - A$ is positive semi-definite. For a vector $X \in \C^d$, denote its transpose by $X^T$ and its conjugate transpose by $X^H$. $\|.\|$ denotes an arbitrary norm on $\R^d$ unless specified. We sometimes denote $\min(a, b)$ by $a \wedge b$ and $\max(a, b)$ by $a \vee b$. For $f: \R^d \rightarrow \R^d$, we denote by $f^k$ the composition of $f$ with itself $k$ times, i.e. $f^k(\omega) = \underbrace{f\circ f\circ \dots \circ f}_{k\ \text{times}}(\omega)$.

\section{Interpretation of spectral quantities in a two-player zero-sum game}\label{section: interpretation spectral quantities}
In this appendix section, we are interested in interpreting spectral bounds in terms of the usual strong convexity and Lipschitz continuity constants in a two-player zero-sum game:
\begin{equation}
\min_{x \in \R^m}\max_{y \in \R^p} f(x, y)
\end{equation}
with $f$ is two times continuously differentiable.

Assume,
\begin{align}
&\mu_1I_m  \preccurlyeq  \nabla_x^2 f \preccurlyeq L_1I_m\\
&\mu_2I_p  \preccurlyeq  -\nabla_y^2 f \preccurlyeq L_2I_p\\
&\mu_{12}^2 I_p \preccurlyeq  (\nabla_x \nabla_y f)^T(\nabla_x \nabla_y f) \preccurlyeq L_{12}^2 I_p
\end{align}
where $\mu_1,\mu_2$ and $\mu_{12}$ are non-negative constants.
Let $\omega^* = (x^*, y^*)$ be a stationary point. To ease the presentation, let,
\begin{equation}\nabla v(\omega^*) = 
\begin{pmatrix}
\nabla_x^2 f(\omega^*) & (\nabla_x \nabla_y f(\omega^*))^T \\
-(\nabla_x \nabla_y f(\omega^*)) & \nabla_y^2 f(\omega^*)\\
\end{pmatrix}
=
\begin{pmatrix}
S_1 & A \\
-A^T & S_2 \\
\end{pmatrix}\,.
\end{equation}
Now, more precisely, we are interested in lower bounding $\Re(\lambda)$ and $|\lambda|$ and upper bounding $|\lambda|$ for $\lambda\in\Sp \nabla v(\omega^*)$.

\subsection{Commutative and square case}\label{subsection: commutative and square case}
In this subsection we focus on the square and commutative case as formalized by the following assumptions:
\begin{asm}[Square and commutative case]\label{asm: square and commutative case}
The following holds:
\begin{enumerate*}[label=(\roman*), font=\itshape]
    \item $p=m=\frac{d}{2}$;
    \item $S_2$ and $A^T$ commute;
    \item $S_1$, $S_2$ and $AA^T$ commute.
\end{enumerate*}
\end{asm}
\Cref{asm: square and commutative case} holds if, for instance, the objective is separable, i.e.
$f(x, y) = \sum_{i = 1}^m f_i(x_i,y_i)$.
Then, using a well-known linear algebra theorem, \cref{asm: square and commutative case} implies that there exists $U \in \R^{d \times d}$ unitary such that $S_1 = U\diag(\alpha_1,\dots,\alpha_m)U^T$, $S_2 = U\diag(\beta_1,\dots,\beta_m)V^T$ and $AA^T = U\diag(\sigma_1^2,\dots,\sigma_{d}^2)U^T$ where $\alpha_1,\dots,\alpha_m$ are the eigenvalues of $S_1$, $\beta_1,\dots,\beta_m$ are the eigenvalues of $S_2$ and $\sigma_1,\dots,\sigma_{p}$ are the singular values of $A$. See \citet[p.~74]{laxLinearAlgebraIts2007} for instance.

Define,
$$ \mu = \begin{pmatrix}
\mu_1 & \mu_{12}\\
-\mu_{12} & \mu_2\\
\end{pmatrix}
$$
$$ L = \begin{pmatrix}
L_1 & L_{12}\\
-L_{12} & L_2\\
\end{pmatrix}\,.
$$
Denote by $|\mu|$ and $|L|$ the determinants of theses matrices, and by $\Tr \mu$ and $\Tr L$ their traces.

In this case we get an exact characterization of the spectrum $\nabla v(\omega^*)$, which we denote by $\sigma^* = \Sp \nabla v(\omega^*)$:
\begin{lemma}\label{lemma: characterization of the spectrum}
Under \cref{asm: square and commutative case}, 
$\lambda \in sigma^*$ if and only if there exists some $i \leq d$ such that $\lambda$ is a root of \begin{align*}P_i = X^2 - (\alpha_i + \beta_i)X + \alpha_i\beta_i + \sigma^2_i \end{align*}
\end{lemma}
\begin{proof}
We compute the characteristic polynomial of $\nabla v(\omega^*)$ using that $S_2$ and $A^T$ commute, using the formula for the determinant of a block matrix, which can be found in \citet[Section 0.3]{zhangSchurComplementIts2005} for instance.
\begin{align*}
\begin{vmatrix}
XI - S_1 & -A\\
A^T & XI - S_2\\
\end{vmatrix}
 &= |(XI - S_1)(XI - S_2) + AA^T|\\
 &= |X^2I - X(S_1 + S_2) + S_1S_2 + AA^T|\\
&= \prod_i \left(X^2 - (\alpha_i + \beta_i)X + \alpha_i\beta_i + \sigma^2_i \right)\\
\end{align*}
\end{proof}

\begin{thm}\label{thm: interpretation of spectral quantities commutative case}
Under \cref{asm: square and commutative case}, we have the following results on the eigenvalues of $\nabla v(\omega^*)$.
\begin{enumerate}[label=(\alph*)]
    \item  For $i \leq m$, if $(\alpha_i - \beta_i)^2 < 4\sigma^2_i$, the roots of $P_i$ satisfy:
    \begin{equation}
        \frac{\Tr \mu}{2} \leq \Re(\lambda),\quad \det \mu \leq |\lambda|^2 \leq \det L\,,\quad \forall \lambda \in \C\ \text{s.t.}\ P_i(\lambda) = 0\,.
    \end{equation}
    \item For $i \leq m$, if $(\alpha_i - \beta_i)^2 \geq 4\sigma^2_i$, the roots of $P_i$ are real non-negative and satisfy :
    \begin{equation}
        \max\left(\mu_1 \wedge \mu_2, \frac{\det \mu}{\Tr L}\right) \leq \lambda \leq L_1 \vee L_2\,,\quad \forall \lambda \in \C\ \text{s.t.}\ P_i(\lambda) = 0\,.
    \end{equation}
    \item Hence, in general, 
    \begin{equation}
          \mu_1\wedge\mu_2 \leq \Re \lambda,\quad  |\lambda|^2 \leq 2L_{max}^2\,              \quad \forall \lambda \in \sigma^*\,,
    \end{equation}
    where $L_{max} = \max(L_1, L_2, L_{12})$.
\end{enumerate}
\end{thm}
\begin{proof}
\begin{enumerate}[label=(\alph*)]
    \item Assume that $(\alpha_i - \beta_i)^2 < 4\sigma^2_i$, i.e.~the discriminant of the polynomial $P_i$ of \Cref{lemma: characterization of the spectrum} is negative.
    Consider $\lambda$ a root of $P_i$. Then $\Re \lambda = \frac{\alpha_i + \beta_i}{2}$ and $|\lambda|^2 = \alpha_i\beta_i + \sigma^2_i$.
    Hence $\Re \lambda \geq \frac{1}{2}\Tr \mu$ and $\det \mu \leq |\lambda|^2 \leq \det L$.
    \item Assume that $(\alpha_i - \beta_i)^2 \geq  4\sigma^2_i$,  i.e.~the discriminant of the polynomial $P_i$ of \Cref{lemma: characterization of the spectrum} is non-negative. This implies that $\Delta =\left(\Tr L \right)^2 - 4\det\mu \geq 0$.

Denote by $\lambda_+$ and $\lambda_-$ the two real roots of $P_i$. Then $$\lambda_\pm = \frac{\alpha_i + \beta_i \pm \sqrt{(\alpha_i + \beta_i)^2 - 4(\alpha_i\beta_i + \sigma_i^2)}}{2}$$
Hence $$ \lambda_+ \geq \lambda_- \geq \min_{\max(\Tr \mu,\ 4\det \mu) \leq x \leq \Tr L} \frac{x - \sqrt{x^2 - 4\det\mu}}{2}$$
As $x \mapsto x - \sqrt{x^2 - 4\det\mu}$ is decreasing on its domain, the minimum is reached at $\Tr L$ and is $\frac{\Tr L - \sqrt{\Delta}}{2} \geq 0$. However this lower bound is quite loose when $A = 0$. So note that 
\begin{align}
\lambda_- &= \frac{\alpha_i + \beta_i - \sqrt{(\alpha_i + \beta_i)^2 - 4(\alpha_i\beta_i + \sigma_i^2)}}{2}\\ &\geq \frac{\alpha_i + \beta_i - \sqrt{(\alpha_i - \beta_i)^2}}{2} = \alpha_i \wedge \beta_i\\ &\geq \mu_1 \wedge \mu_2\\
\end{align}
Similarly,
\begin{equation}
\lambda_+ \leq \frac{\alpha_i + \beta_i + \sqrt{\left(\alpha_i - \beta_i \right)^2}}{2} = \alpha_i \vee \beta_i \leq L_1 \vee L_2 
\end{equation}

Finally:
\begin{equation}
L_1 \vee L_2 \geq \lambda_+ \geq \lambda_- \geq \max\left(\frac{\Tr L - \sqrt{\Delta}}{2}, \mu_1 \wedge \mu_2\right)\,.
\end{equation}
Moreover, \begin{align}
  \Tr L - \sqrt{\Delta} &=  \frac{\left(  \Tr L - \sqrt{\Delta}\right)\left(  \Tr L + \sqrt{\Delta}\right)}{  \Tr L + \sqrt{\Delta}}\\
  &= \frac{4\det\mu}{\Tr L + \sqrt{\Delta}}\\
  &\geq \frac{2\det\mu}{\Tr L}\,,
\end{align}
which yields the result.
\item These assertions are immediate corollaries of the two previous ones.
\end{enumerate}
\end{proof}

We need the following lemma to be able to interpret \Cref{thm: global rate eg} in the context of \cref{ex: highly adversarial game}, whose assumptions imply \cref{asm: square and commutative case}.
\begin{lemma}\label{lemma: singular values highly adv}
Under \cref{asm: square and commutative case}, the singular values of $\nabla v(\omega^*)$ can be lower bounded as:
\begin{equation}
\mu_{12}(\mu_{12} - \max(L_1 - \mu_2, L_2 - \mu_1)) \leq \sigma_{min}(\nabla v(\omega^*))^2\,.
\end{equation}
In particular, if $\mu_{12} > 2\max(L_1-\mu_2, L_2-\mu_1)$, this becomes 
\begin{equation}
\half\mu_{12}^2\leq \sigma_{min}(\nabla v(\omega^*))^2\,.
\end{equation}
\end{lemma}
\begin{proof}
To prove this we compute the eigenvalues of $(\nabla v(\omega^*))^T\nabla v(\omega^*)$.
We have that,
\begin{equation}
(\nabla v(\omega^*))^T\nabla v(\omega^*) = \begin{pmatrix}
S_1^2 + AA^T & S_1A-AS_2\\
A^TS_1 - S_2A^T & A^TA + S_2^2\\
\end{pmatrix}\,.
\end{equation}
As in the proof of \Cref{lemma: characterization of the spectrum}, as \cref{asm: square and commutative case} implies that $A^TS_1 - S_2A^T$ and $A^TA + S_2^2$ commute,
\begin{align}
|XI-(\nabla v(\omega^*))^T\nabla v(\omega^*)| &= \left|(XI - S_1^2 - AA^T)(XI - S_2^2 - A^TA)-(S_1 - S_2)^2AA^T\right|\\
&= \prod_i\left((XI - \alpha_i^2 - \sigma_i^2)(XI - \beta_i^2 - \sigma_i^2) - (\alpha_i - \beta_i)^2\sigma_i^2\right)\,.
\end{align}
Let $Q_i(X) = (XI - \alpha_i^2 - \sigma_i^2)(XI - \beta_i^2 - \sigma_i^2) - (\alpha_i - \beta_i)^2\sigma_i^2$. Its discriminant is \begin{align}
\Delta_i' &= (\alpha_i^2 + \beta_i^2 + 2\sigma_i^2)^2 - 4\left((\alpha_i^2 + \sigma_i^2)(\beta_i^2 + \sigma_i^2) - (\alpha_i - \beta_i)^2)\sigma_i^2\right)\\
&= (\alpha_i - \beta_i)^2((\alpha_i+\beta_i)^2 + 4\sigma_i^2) \geq 0\,.
\end{align}
Hence the roots of $Q_i$ are:
\begin{equation}
\lambda_{i\pm} = \half\left(\alpha_i^2 + \beta_i^2 + 2\sigma_i^2 \pm \sqrt{ (\alpha_i - \beta_i)^2((\alpha_i+\beta_i)^2 + 4\sigma_i^2) }\right)\,.
\end{equation}
The smallest is $\lambda_{i-}$ which can be lower bounded by
\begin{align}
\lambda_{i-} &= \half\left(\alpha_i^2 + \beta_i^2 + 2\sigma_i^2 - \sqrt{ (\alpha_i + \beta_i)^2((\alpha_i+\beta_i)^2 + 4\sigma_i^2)}\right)\\
&\geq \half\left(\alpha_i^2 + \beta_i^2 - |\alpha_i^2 - \beta_i^2| + 2\sigma_i(\sigma_i - |\alpha_i - \beta_i|)\right)\\
&\geq \sigma_i(\sigma_i - |\alpha_i - \beta_i|)\\
&\geq \mu_{12}(\mu_{12} - \max(L_1 - \mu_2, L_2 - \mu_1))\,.
\end{align}
\end{proof}
 
\section{Complement for \cref{section: game optimization}}\label{section: complement background}
The convergence result of \Cref{thm: local convergence} can be strengthened if the Jacobian is constant as shown below.
 A proof of this classical result in linear algebra can be found in \cite{arjevaniLowerUpperBounds2016} for instance.

\begin{thm}\label{thm: global convergence for constant jacobian}
Let $F: \R^d \longrightarrow \R^d$ be a linear operator. If $\rho(\nabla F) < 1$, then for all $\omega_0 \in \R^d$, the iterates $(\omega_t)_t$ defined as above
converge linearly to $\omega^*$ at a rate of $\bigO((\rho(\nabla F))^t)$.
\end{thm}
\section{Convergence results of \S\ref{section: revisiting}}

 Let us restate \Cref{thm: spectral rate for gradient} for clarity.

\thmSpectralRateGradient*
In this subsection, we quickly show how to obtain $(i)$ of \Cref{thm: spectral rate for gradient} from Theorem 2 of \citet{gidelNegativeMomentumImproved2018b}, whose part which interests us now is the following:
\begin{thm*}[{\citet[part of Theorem 2]{gidelNegativeMomentumImproved2018b}}]
If the eigenvalues of $\nabla v(\omega^*)$ all have positive real parts, then for $\eta=\Re(1/\lambda_1)$ one has
\begin{equation}
\rho(\nabla F_\eta(\omega^*))^2 \leq 1 -  \Re(1/\lambda_1) \delta
\end{equation}
where $\delta = \min_{1 \leq j \leq m} |\lambda_j|^2(2 \Re(1/\lambda_j) - \Re(1/\lambda_1))$ and $\Sp \nabla v(\omega^*) = \{\lambda_1,\dots,\lambda_m\}$ sorted such that $0 < \Re(1/\lambda_1) \leq \Re(1/\lambda_2) \leq \dots \leq \Re(1/\lambda_m)$.
\end{thm*}
\begin{proof}[Proof of $(i)$ of \Cref{thm: spectral rate for gradient}]
By definition of the order on the eigenvalues,
\begin{align}
\delta &= \min_{1 \leq j \leq m} |\lambda_j|^2(\Re(1/\lambda_j) + \Re(1/\lambda_j) - \Re(1/\lambda_1))\\
&\geq \min_{1 \leq j \leq m} |\lambda_j|^2(\Re(1/\lambda_j))\\
&= \min_{1 \leq j \leq m} \Re(\lambda_j)
\end{align}
\end{proof}
To prove the second part of \Cref{thm: spectral rate for gradient}, we rely on a different part of \citet[Theorem 2]{gidelNegativeMomentumImproved2018b} which we recall below:
 \begin{thm*}[{\citet[part of Theorem 2]{gidelNegativeMomentumImproved2018b}}]
 The best step-size $\eta^*$, that is to say the solution of the optimization problem
 \begin{equation}
    \min_{\eta}\rho(\nabla F_\eta(\omega^*))^2\,, 
 \end{equation}
 satisfy:
 \begin{equation}
\min_{\lambda \in \sigma^*} \Re(1/\lambda) \leq \eta^* \leq 2\min_{\lambda \in \sigma^*} \Re(1/\lambda)\,.
 \end{equation}
 \end{thm*}
$(ii)$ of \cref{thm: spectral rate for gradient} is now immediate.
\begin{proof}[Proof of $(ii)$ of \cref{thm: spectral rate for gradient}]
By definition of the spectral radius, 
\begin{align}
\rho(\nabla F_{\eta^*}(\omega^*))^2 &= \max_{\lambda \in \Sp(\nabla v(\omega^*)} |1 - \eta^* \lambda|^2\\
&= 1 - \min_{\lambda \in \Sp(\nabla v(\omega^*)} 2\eta^* \Re \lambda - |\eta^* \lambda|^2\\
&\geq  1 - \min_{\lambda \in \Sp(\nabla v(\omega^*)} 2\eta^* \Re \lambda\\
&\geq 1 - 4\min_{\lambda \in \Sp(\nabla v(\omega^*)} \Re \lambda \min_{\lambda \in \sigma^*} \Re(1/\lambda)
\end{align}
\end{proof}
\corInterpretedRateOfGradient*
\begin{proof}
Note that the hypotheses stated in \S\ref{subsection: comparison} correspond to the assumptions of \S\ref{subsection: commutative and square case}. Moreover, with the notations of this subsection, one has that $4 \sigma_i^2 \geq 4\mu_{12}^2$ and $\max(L_1, L_2)^2 \geq (\alpha_i - \beta_i)^2$. Hence the condition $2\mu_{12} \geq \max(L_1, L_2)$ implies that all the eigenvalues of $\nabla v(\omega^*)$ satisfy the case $(a)$ of \Cref{thm: interpretation of spectral quantities commutative case}. Then, using \Cref{thm: spectral rate for gradient}, 
\begin{align}
\rho(\nabla F_\eta(\omega^*))^2 &\leq 1 -  \min_{\lambda \in \sigma^*} \Re(1/\lambda) \min_{\lambda \in \sigma^*} \Re(\lambda)\\
&\leq  1 -  \left(\frac{\min_{\lambda \in \sigma^*} \Re(\lambda)}{\max_{\lambda \in \sigma^*} |\lambda|}\right)^2\\
&\leq  1 - \frac{1}{4}\frac{(\mu_1 + \mu_2)^2}{L_{12}^2 + L_1L_2} \,.
\end{align}
\end{proof}
\section{Spectral analysis of \S\ref{section: convergence analysis}}\label{subsection: app spectral analysis of eg} %
We prove \cref{lemma: spectra operators}.
\begin{lemma}
Assuming that the eigenvalues of $\nabla v(\omega^*)$ all have non-negative real parts, the proximal point operator $P_\eta$ is continuously differentiable in a neighborhood of $\omega^*$ . Moreover, the spectra of the $k$-extrapolation operator and the proximal point operator are given by:
\begin{align}
&\Sp \nabla F_{\eta, k}(\omega^*) = \big\{ \textstyle\sum_{j = 0}^k (-\eta \lambda)^j\ |\ \lambda \in \sigma^*\big\} \\
& \text{and}\ \ 
\Sp \nabla P_\eta(\omega^*) = \left\{ (1 + \eta \lambda)^{-1}\ |\ \lambda \in \sigma^*\right\}\,.
\end{align}
Hence, for all $\eta > 0$, the spectral radius of the operator of the proximal point method is equal to: \begin{equation}
\rho(\nabla P_{\eta}(\omega^*))^2 = 1 - \min_{\lambda \in \sigma^*} \tfrac{2\eta \Re \lambda + \eta^2|\lambda|^2}{|1+\eta \lambda|^2}\,.
\end{equation}
\end{lemma}
To prove the result about the $k$-extrapolation operator, we first show the following lemma, which will be used again later.

Recall that we defined $\varphi_{\eta, \omega}: z \mapsto \omega - \eta v(z)$. We drop the dependence on $\eta$ in $\varphi_{\eta, \omega}$ for compactness.
\begin{lemma}\label{lemma: jacobians of extrapolation operators}
The Jacobians of $\varphi_\omega^k(z)$ with respect to $z$ and $\omega$ can be written as 
\begin{align}
    \nabla_z \varphi_\omega^k(z) &= (-\eta)^k \nabla v(\varphi_\omega^{k-1}(z))\nabla v(\varphi_\omega^{k-2}(z))\dots\nabla v(\varphi_\omega^{0}(z))\\
    \nabla_\omega \varphi_\omega^k(z) &= \sum_{j=0}^{k-1} (-\eta)^{j} \nabla v(\varphi_\omega^{k-1}(z))\nabla v(\varphi_\omega^{k-2}(z))\dots\nabla v(\varphi_\omega^{k-j}(z))\,.
\end{align}
\end{lemma}
\begin{proof}
We prove the result by induction:
\begin{itemize}
    \item For $k = 1$, $\varphi_\omega(z) = \omega - \eta v(z)$ and the result holds.
    \item Assume this result holds for $k \geq 0$. Then, 
    \begin{align}
        \nabla_z \varphi_\omega^{k+1}(z) &= \nabla_z \varphi_\omega(\varphi_\omega^k(z))\nabla_z \varphi_\omega^{k}(z)\\
        &=-\eta \nabla v(\varphi_\omega^k(z)) (-\eta)^k \nabla v(\varphi_\omega^{k-1}(z))\dots\nabla v(\varphi_\omega^{0}(z))\\
        &=(-\eta)^{k+1} \nabla v(\varphi_\omega^{k}(z))\nabla v(\varphi_\omega^{k-1}(z))\dots\nabla v(\varphi_\omega^{0}(z))\,.
    \end{align}
    For the derivative with respect to $\omega$, we use the chain rule:
    \begin{align}
    \nabla_\omega\varphi_\omega^{k+1}(z) &= \nabla_\omega \varphi_\omega(\varphi_\omega^k(z)) + \nabla_z \varphi_\omega(\varphi_\omega^k(z))\nabla_\omega\varphi_\omega^{k}(z)\\
    &= I_d -\eta v(\varphi_\omega^k(z))\sum_{j=0}^{k-1} (-\eta)^{j} \nabla v(\varphi_\omega^{k-1}(z))\dots\nabla v(\varphi_\omega^{k-j}(z))\\
    &= I_d + \sum_{j=0}^{k-1} (-\eta)^{j+1} \nabla v(\varphi_\omega^{k}(z))\nabla v(\varphi_\omega^{k-1}(z))\dots\nabla v(\varphi_\omega^{k-j}(z))\\
    &= I_d + \sum_{j=1}^{k} (-\eta)^{j} \nabla v(\varphi_\omega^{k}(z))\nabla v(\varphi_\omega^{k-1}(z))\dots\nabla v(\varphi_\omega^{k+1-j}(z))\\
    &=\sum_{j=0}^{k} (-\eta)^{j} \nabla v(\varphi_\omega^{k}(z))\nabla v(\varphi_\omega^{k-1}(z))\dots\nabla v(\varphi_\omega^{k+1-j}(z))
    \end{align}
\end{itemize}
\end{proof}
In the proof of \Cref{lemma: spectra operators} and later we will use the spectral mapping theorem, which we state below for reference:
\begin{thm}[Spectral Mapping Theorem]\label{thm: spectral mapping theorem}
Let $A \in \mathbb{C}^{d\times d}$ be a square matrix, and $P$ be a polynomial. Then,
\begin{equation}
\Sp P(A) = \{P(\lambda)\ |\ \lambda \in \Sp A\}\,.    
\end{equation}
\end{thm}
See for instance \citet[Theorem 4, p.~66 ]{laxLinearAlgebraIts2007} for a proof.
\begin{proof}[Proof of \Cref{lemma: spectra operators}]
First we compute $\nabla F_{\eta, k}(\omega^*)$. As $\omega^*$ is a stationary point, it is a fixed point of the extrapolation operators, i.e.~$\varphi_{\omega^*}^j(\omega^*) = \omega^*$ for all $j \geq 0$. Then, by the chain rule, 
\begin{align}
\nabla F_{\eta, k}(\omega^*) &= \nabla_z \varphi_{\omega^*}^k(\omega^*) + \nabla_\omega \varphi_{\omega^*}^k(\omega^*)\\
&= (-\eta \nabla v(\omega^*))^k +  \sum_{j=0}^{k-1} (-\eta \nabla v(\omega^*))^{j}\\
&=  \sum_{j=0}^{k} (-\eta \nabla v(\omega^*))^{j}\,.
\end{align}
Hence $\nabla F_{\eta, k}(\omega^*)$ is a polynomial in $\nabla v(\omega^*)$.
Using the spectral mapping theorem (\Cref{thm: spectral mapping theorem}), one gets that 
\begin{equation}
\Sp \nabla F_{\eta, k}(\omega^*) = \left\{ \sum_{j = 0}^k (-\eta)^j \lambda^j\ |\ \lambda \in \Sp \nabla v(\omega^*)\right\}
\end{equation}
For the proximal point operator, first let us prove that it is differentiable in a neighborhood of $\omega^*$. First notice that,
\begin{equation}
\Sp(I_d + \eta \nabla v(\omega^*)) = \{1 + \eta \lambda\ |\ \lambda \in \Sp \nabla v(\omega^*)\}\,.
\end{equation}
If the eigenvalues of $\nabla v(\omega^*)$ all have non-negative real parts, this spectrum does not contain zero. Hence $\omega \mapsto \omega + \eta v(\omega)$ is continuously differentiable and has a non-singular differential at $\omega^*$. By the inverse function theorem (see for instance \citet{rudinPrinciplesMathematicalAnalysis1976}), $\omega \mapsto \omega + \eta v(\omega)$ is invertible in a neighborhood of $\omega^*$ and its inverse, which is $P_\eta$, is continuously differentiable there. Moreover, 
\begin{equation}\label{eq: inverse proximal}
\nabla P_\eta(\omega^*) = (I_d + \eta \nabla v(\omega^*))^{-1}\,.
\end{equation}
Recall that the eigenvalues of a non-singular matrix are exactly the inverses of the eigenvalues of its inverse. Hence,
\begin{equation}
    \Sp \nabla P_\eta(\omega^*) = \{\lambda^{-1}\ |\ \lambda \in \Sp(I_d + \eta \nabla v(\omega^*)) \} = \left\{ (1 + \eta \lambda)^{-1}\ |\ \lambda \in \Sp \nabla v(\omega^*)\right\}\,,
\end{equation}
where the last equality follows from the spectral mapping theorem applied to $I_d + \eta \nabla v(\omega^*)$. Now, the bound on the spectral radius of the proximal point operator is immediate.
Indeed, its spectral radius is:
\begin{align}
\rho(\nabla P_\eta(\omega^*))^2 &= \max_{\lambda \in \sigma^*} \frac{1}{|1+\eta\lambda|^2}\\
&=   1 - \min_{\lambda \in \sigma^*}\left(\frac{2\eta \Re \lambda + \eta^2|\lambda|^2}{|1+\eta \lambda|^2}\right)
\end{align}
which yields the result.
\end{proof}
\thmSpectralRateExtragradient*
\begin{proof}
Let $L = \max_{\lambda \in \sigma^*} |\lambda|$ and $\eta = \frac{\tau}{L}$ for some $\tau > 0$.
For $\lambda \in \sigma^*$, 
\begin{align}
\left|\sum_{j = 0}^k (-\eta)^j \lambda^j\right|^2 &= \frac{|1 - (-\eta)^{k+1}\lambda^{k+1}|^2}{|1+\eta \lambda|^2}\\
&=\frac{1+2(-1)^k\eta^{k+1}\Re(\lambda^{k+1}) + \eta^{2(k+1)}|\lambda|^{2(k+1)}}{|1+\eta \lambda|^2}\\
&=1 - \frac{2\eta \Re \lambda+\eta^2 |\lambda|^2 - 2(-1)^k\eta^{k+1}\Re(\lambda^{k+1}) - \eta^{2(k+1)}|\lambda|^{2(k+1)}}{|1+\eta \lambda|^2}\\
&= 1 - \frac{2\eta \Re \lambda+\eta^2 |\lambda|^2\left(1 - 2(-1)^k\eta^{k-1}\frac{\Re(\lambda^{k+1})}{|\lambda|^2} - \eta^{2(k-1)}|\lambda|^{2(k-1)}\right)}{|1+\eta \lambda|^2}
\end{align}
Now we focus on lower bounding the terms in between the parentheses. By definition of $\eta$, we have
$\eta^{k-1}\frac{|\Re(\lambda^{k+1})|}{|\lambda|^2} \leq \tau^{k-1}$ and $\eta^{2(k-1)}|\lambda|^{2(k-1)} \leq \tau^{2(k-1)}$. Hence
\begin{align}
1 + 2(-1)^k\eta^{k-1}\frac{\Re(\lambda^{k+1})}{|\lambda|^2} + \eta^{2(k-1)}|\lambda|^{2(k-1)})&\geq 1 - 2\eta^{k-1}\frac{|\Re(\lambda^{k+1})|}{|\lambda|^2} - \eta^{2(k-1)}|\lambda|^{2(k-1)}\\
&\geq 1 - 2\tau^{k-1} - \tau^{2(k-1)}\\
\end{align}
Notice that if $k=1$, i.e. for the gradient method, we cannot control this quantity. However, for $k\geq 2$, if $\tau \leq (\frac{1}{4})^{\frac{1}{k-1}}$, one gets that \begin{align}
1 - 2\tau^{k-1} - \tau^{2(k-1)} \geq 1 - \frac{1}{2} - \frac{1}{16} = \frac{7}{16}
\end{align}
which yields the first assertion of the theorem.
For the second one, take $\eta = \frac{1}{4L}$, i.e. the maximum step-size authorized for extragradient, and one gets that
\begin{align}
|1+\eta \lambda|^2 &= 1 + 2\eta \Re \lambda +  \eta^2 |\lambda|^2\\
&\leq 1 + 2\frac{1}{4} + \frac{1}{16} = \frac{25}{16}\,.
\end{align}
Then,
\begin{align}
 \frac{2\eta \Re \lambda+ \frac{7}{16}\eta^2 |\lambda|^2}{|1+\eta \lambda|^2} &\geq \frac{1}{4}\left(2\frac{16}{25} \frac{\Re \lambda}{L} + \frac{7}{100}\frac{|\lambda|^2}{L^2}\right)\\
 &\geq \frac{1}{4}\left(\frac{\Re \lambda}{L} + \frac{7}{112}\frac{|\lambda|^2}{L^2}\right)\\
 &\geq \frac{1}{4}\left(\frac{\Re \lambda}{L} + \frac{1}{16}\frac{|\lambda|^2}{L^2}\right)\,,
\end{align}
which yields the desired result.
\end{proof}

\corRateEGBilinearGame*
First we need to compute the eigenvalues of $\nabla v$.
\begin{lemma}\label{lemma: spectrum bilinear game}
Let $A \in \R^{m \times m}$ and 
\begin{equation}
    M = \begin{pmatrix}
    0_m & A\\
    -A^T & 0_m\\
    \end{pmatrix}\,.
\end{equation}
Then,
\begin{equation}
\Sp M = \{ \pm i\sigma\ |\ \sigma^2 \in \Sp AA^T\}\,.
\end{equation}
\end{lemma}
\begin{proof}
\Cref{asm: square and commutative case} of \Cref{subsection: commutative and square case} holds so we can apply \Cref{lemma: characterization of the spectrum} which yields the result.

\end{proof}
\begin{proof}[Proof of \Cref{cor: rate of eg for bilinear game}]
The Jacobian is constant here and has following the form:
\begin{equation}
    \nabla v = \begin{pmatrix}
    0_m & A\\
    -A^T & 0_m\\
    \end{pmatrix}\,.
\end{equation}
Applying \Cref{lemma: spectrum bilinear game} yields
\begin{equation}
\Sp \nabla v = \{ \pm i\sigma\ |\ \sigma^2 \in \Sp AA^T\}\,.
\end{equation}
Hence $\min_{\lambda \in \Sp \nabla v} |\lambda|^2 = \sigma_{min}(A)^2 $ and $\max_{\lambda \in \Sp \nabla v} |\lambda|^2 = \sigma_{max}(A)^2$. Using \Cref{thm: rate of n-extrapolation}, we have that,
\begin{equation}
    \rho(\nabla F_{\eta,k}(\omega^*))^2 \leq \left(1 - \frac{1}{64}\frac{\sigma_{min}(A)^2}{{\sigma_{max}(A)^2}}\right)\,.
\end{equation}
Finally, \Cref{thm: global convergence for constant jacobian} implies that the iterates of the $k$-extrapolation converge globally at the desired rate.
\end{proof}

\begin{cor}\label{cor: interpreted rate of extrgradient}
Under the assumptions of \Cref{cor: interpreted rate of gradient}, the spectral radius of the n-extrapolation method operator is bounded by 
\begin{equation}\label{eq: interpreted rate eg}
\rho(\nabla F_{\eta, k}(\omega^*))^2 \leq 1 - \frac{1}{4}\Big(\frac{1}{2}\frac{\mu_1 + \mu_2}{\sqrt{L_{12}^2 + L_1L_2}}+\frac{1}{16}\frac{\mu_{12}^2 + \mu_1\mu_2}{L_{12}^2 + L_1L_2}\Big)\,.
\end{equation}
\end{cor}
\begin{proof}
This is a direct consequence of \Cref{thm: rate of n-extrapolation} and \Cref{thm: interpretation of spectral quantities commutative case}, as the latter gives that for any $\lambda \in \Sp \nabla v(\omega^*)$, \begin{equation}
        \frac{\Tr \mu}{2} \leq \Re \lambda,\quad |\mu| \leq |\lambda|^2 \leq |L|\,,
    \end{equation}
    as discussed in the proof of \Cref{cor: interpreted rate of gradient}.
\end{proof}
\section{Global convergence proofs}
In this section, $\|.\|$ denotes the Euclidean norm.
\subsection{Alternative characterizations and properties of the assumptions}\label{subsection: alternative charact}
\lemmaWeakStrongMonotonicity*
Let us recall \cref{eq: weak strong monotonicity} here for simplicity:
\begin{equation}\label{eq: lemma weak strong monotonicity}
\| \omega - \omega'\| \leq \gamma^{-1} \|v(\omega) - v(\omega')\|\quad  \forall \omega,\omega' \in \R^d\,.\tag{\ref{eq: weak strong monotonicity}}  
\end{equation}
The proof of this lemma is an immediate consequence of a  global inverse theorem from \citet{jacqueshadamardTransformationsPonctuelles1906,m.p.levyFonctionsLignesImplicites1920}. Let us recall its statement here:
\begin{thm}[\citet{jacqueshadamardTransformationsPonctuelles1906,m.p.levyFonctionsLignesImplicites1920}]\label{thm: Hadamard}
Let $f:\R^d \rightarrow \R^d$ be a continuously differentiable map. Assume that, for all $\omega \in \R^d$, $\nabla f$ is non-singular and $\sigma_{min}(\nabla f) \geq \gamma > 0$. Then $f$ is a $C^1$-diffeomorphism, i.e. a one-to-one map whose inverse is also continuously differentiable.
\end{thm}
A proof of this theorem can be found in \citet[Theorem 3.11]{rheinboldtLocalMappingRelations1969}.
We now proceed to prove the lemma.
\begin{proof}[Proof of \Cref{lemma: weak strong monotonicity}]
First we prove the direct implication. By the theorem stated above, $v$ is a bijection from $\R^d$ to $\R^d$, its inverse is continuously differentiable on $\R^d$ and so we have, for all $\omega \in \R^d$:
\begin{equation}
\nabla v^{-1}(v(\omega)) = (\nabla v(\omega))^{-1}\, .
\end{equation}
Hence $\|\nabla v^{-1}(v(\omega))\| = (\sigma_{min}(\nabla v(\omega)))^{-1} \leq \gamma^{-1}$.

Consider $\omega, \omega' \in \R^d$ and let $u = v(\omega)$ and $u' = v(\omega')$. Then
\begin{align}
\| \omega - \omega' \| &= \| v^{-1}(u) - v^{-1}(u') \|\\
&= \left\Vert \int_{0}^1 \nabla v^{-1}(tu + (1-t)u') (u - u') \right\Vert \\
&\leq \gamma^{-1} \|u - u'\|\\
&= \gamma^{-1} \|v(\omega) - v(\omega')\|
\end{align}
which proves the result.

Conversely, if \cref{eq: lemma weak strong monotonicity} holds, fix $u \in \R^d$ with $\|u\| = 1$. Taking $\omega' = \omega + tu$ in \cref{eq: lemma weak strong monotonicity} with $t \neq 0$ and rearranging yields:
$$\gamma \leq \left\Vert\frac{v(\omega+tu) - v(\omega)}{t}\right\Vert\,.$$
Taking the limit when $t$ goes to $0$ gives that $\gamma \leq \|\nabla v(\omega)u\|$. As it holds for all $u$ such that $\|u\|=1$ this implies that $\gamma \leq \sigma_{min}(\nabla v)$.
\end{proof}
With the next lemma, we relate the quantities appearing in \Cref{thm: global rate eg} to the spectrum of $\nabla v $. Note that the first part of the proof is standard --- it can be found in \citet[Prop.~2.3.2]{facchineiFiniteDimensionalVariationalInequalities2003} for instance --- and we include it only for completeness.
 \begin{lemma}\label{lemma: relation global quantities to spectrum}
Let $v : \R^d \rightarrow \R^d$ be continuously differentiable and \begin{enumerate*}[series = tobecont, itemjoin = \quad, label=(\roman*)] \item $\mu$-strongly monotone for some $\mu \geq 0$, \item $L$-Lispchitz, \item such that $\sigma_{min}(\nabla v) \geq \gamma$ for some $\gamma \geq 0$.\end{enumerate*}. Then, for all $\omega \in \R^d$,
 \begin{equation}
\mu\|u\|^2 \leq (\nabla v(\omega) u)^Tu\,,\quad \gamma\|u\| \leq \|\nabla v(\omega) u\| \leq L\|u\|\,,\quad \forall u \in \R^d\,,
 \end{equation}
 and
\begin{equation}
        \mu \leq \Re(\lambda),\quad \gamma \leq |\lambda| \leq L\,, \quad \forall \lambda \in \Sp \nabla v(\omega)\,.
\end{equation}
 \end{lemma}
 \begin{proof}
By definition of $\mu$-strong monotonicity, and $L$-Lispchitz one has that, for any $\omega, \omega' \in \R^d$, 
\begin{align}
 &\mu\|\omega - \omega'\|^2 \leq  (v(\omega) - v(\omega'))^T(\omega - \omega')\\
 &\| v(\omega) - v(\omega')\| \leq L \|\omega - \omega'\|\,.
\end{align}
Fix $\omega \in \R^d$, $u \in \R^d$ such that $\|u\|=1$. Taking $\omega' = \omega + t u$ for $t> 0$ in the previous inequalities and dividing by $t$ yields
\begin{align}
&\mu \leq  \frac{1}{t}(v(\omega) - v(\omega + t u))^Tu\\
&\frac{1}{t}\| v(\omega) - v(\omega+tu)\| \leq L\,.
\end{align}
Letting $t$ goes to 0 gives
\begin{align}
&\mu \leq (\nabla v(\omega) u)^Tu\\
&\|\nabla v(\omega)u\| \leq L\,.
\end{align}

 Furthermore, by the properties of the singular values, 
 \begin{equation}
     \|\nabla v(\omega) u\| \geq \gamma\,.
 \end{equation}
 Hence, by homogeneity, we have that, for all $u \in \R^d$,
 \begin{equation}\label{eq: role mu gamma L}
\mu\|u\|^2 \leq (\nabla v(\omega) u)^Tu\,,\quad \gamma\|u\| \leq \|\nabla v(\omega) u\| \leq L\|u\|\,.
 \end{equation}
 Now, take $\lambda \in \Sp \nabla v(\omega)$ an eigenvalue of $\nabla v(\omega)$ and let $Z \in \C^d\setminus\{0\}$ be one of its associated eigenvectors. Note that $Z$ can be written as $Z = X +iY$ with $X, Y \in \R^d$. By definition of $Z$, we have
 \begin{equation}
     \nabla v(\omega) Z = \lambda Z\,.
 \end{equation}
 Now, taking the real and imaginary part yields:
 \begin{equation}\label{eq: projection eigenvector}
     \begin{cases}
\nabla v(\omega) X &= \Re(\lambda) X - \Im(\lambda) Y\\
\nabla v(\omega) Y &= \Im(\lambda) X + \Re(\lambda) Y
\end{cases}
 \end{equation}
Taking the squared norm and developing the right-hand sides yields
 \begin{equation}
\begin{cases}
\|\nabla v(\omega)X\|^2 &= \Re(\lambda)^2 \|X\|^2 + \Im(\lambda)^2 \|Y\|^2 - 2\Re(\lambda)\Im(\lambda)X^TY\\
\|\nabla v(\omega) Y\|^2 &= \Im(\lambda)^2 \|X\|^2 + \Re(\lambda)^2 \|Y\|^2 + 2\Re(\lambda)\Im(\lambda)X^TY\,.
\end{cases}
 \end{equation}
 Now summing these two equations gives
 \begin{equation}
    \|\nabla v(\omega)X\|^2 + \|\nabla v(\omega) Y\|^2 = |\lambda|^2(\|X\|^2 + \|Y\|^2)\,.
 \end{equation}
 Finally, apply \cref{eq: role mu gamma L} for $u = X$ and $u = Y$:
 \begin{equation}
      \gamma^2(\|X\|^2 + \|Y\|^2)\leq |\lambda|^2(\|X\|^2 + \|Y\|^2)\leq L^2(\|X\|^2 + \|Y\|^2)\,.
 \end{equation}
 As $Z \neq 0$, $\|X\|^2 + \|Y\|^2 > 0$ and this yields $\gamma \leq |\lambda| \leq L$. 
 To get the inequality concerning $\gamma$, multiply on the left the first line of \cref{eq: projection eigenvector} by $X^T$ and the second one by $Y^T$:
  \begin{equation}
     \begin{cases}
X^T(\nabla v(\omega) X) &= \Re(\lambda) \|X\|^2 - \Im(\lambda) X^TY\\
Y^T(\nabla v(\omega) Y) &= \Im(\lambda) Y^TX + \Re(\lambda) \|Y\|^2\,.
\end{cases}
 \end{equation}
 Again, summing these two lines and using \cref{eq: role mu gamma L} yields:
 \begin{equation}
     \mu(\|X\|^2 + \|Y\|^2) \leq \Re(\lambda)(\|X\|^2 + \|Y\|^2)\,.
 \end{equation}
 As $Z \neq 0$, $\|X\|^2 + \|Y\|^2 > 0$ and so $\mu \leq \Re(\lambda)$.
 \end{proof}
\subsection{Proofs of \S\ref{section: global cv rate}: extragradient, optimistic and proximal point methods}\label{app:proof}

We now prove a slightly more detailed version of \cref{thm: global rate eg}.
\begin{thm}\label{thm: app global rate eg}
Let $v : \R^d \rightarrow \R^d$ be continuously differentiable and \begin{enumerate*}[series = tobecont, itemjoin = \quad, label=(\roman*)] \item $\mu$-strongly monotone for some $\mu \geq 0$, \item $L$-Lipschitz, \item such that $\sigma_{min}(\nabla v) \geq \gamma$ for some $\gamma > 0$.\end{enumerate*} Then, for $\eta \leq (4L)^{-1}$, the iterates of the extragradient method $(\omega_t)_t$ converge linearly to $\omega^*$ the unique stationary point of $v$,
\begin{equation}
\|\omega_t - \omega^*\|_2^2 \leq \left( 1 - \left(\eta\mu + \tfrac{7}{16}\eta^2\gamma^2\right)\right)^t \|\omega_0 - \omega^*\|_2^2\,.
\end{equation}
For $\eta = (4L)^{-1}$, this can be simplified as: $\|\omega_t - \omega^*\|_2^2 \leq \left( 1 - \frac{1}{4}\left(\frac{\mu}{L} + \frac{1}{16}\frac{\gamma^2}{L^2}\right)\right)^t \|\omega_0 - \omega^*\|_2^2$. %
\end{thm}
The proof is inspired from the ones of \citet{gidelVariationalInequalityPerspective2018a, tsengLinearConvergenceIterative1995}.

We will use the following well-known identity. It can be found in \citet{gidelVariationalInequalityPerspective2018a} for instance but we state it for reference.
\begin{lemma}\label{lemma: cvx analysis}
Let $\omega, \omega', u \in \R^d$. Then 
\begin{equation}
\|\omega + u - \omega'\|^2 = \|\omega - \omega'\|^2 + 2u^T(\omega+u - \omega') - \|u\|^2
\end{equation}

\end{lemma}
\begin{proof}
\begin{align}
 \|\omega + u - \omega'\|^2 &= \|\omega - \omega'\|^2 + 2u^T(\omega - \omega') + \|u\|^2 \\
 &= \|\omega - \omega'\|^2 + 2u^T(\omega+u - \omega') - \|u\|^2
\end{align}
\end{proof}

\begin{proof}[Proof \Cref{thm: app global rate eg}]
First note that as $\gamma > 0$, by \cref{thm: Hadamard}, $v$ has a stationary point $\omega^*$ and it is unique.

Fix any $\omega_0 \in  \R^d$, and denote $\omega_1 = \omega_0 - \eta v(\omega_0)$ and $\omega_2 = \omega_0 - \eta v(\omega_1)$.
Applying \Cref{lemma: cvx analysis} for $(\omega, \omega', u) = (\omega_0, \omega^*, -\eta v(\omega_1))$ and $(\omega, \omega', u) = (\omega_0, \omega_2, -\eta v(\omega_0))$ yields:
\begin{align}
\|\omega_2 - \omega^*\|^2 &= \|\omega_0 - \omega^*\|^2 - 2\eta v(\omega_1)^T(\omega_2 - \omega^*) - \|\omega_2 - \omega_0\|^2\\
\|\omega_1 - \omega_2\|^2 &= \|\omega_0 - \omega_2\|^2 - 2\eta v(\omega_0)^T(\omega_1 - \omega_2) - \|\omega_1 - \omega_0\|^2
\end{align}
Summing these two equations gives:
\begin{align}\label{eq: sum eq global cv proof}
&\|\omega_2 - \omega^*\|^2 =\\&\|\omega_0 - \omega^*\|^2 - 2\eta v(\omega_1)^T(\omega_2 - \omega^*) - 2\eta v(\omega_0)^T(\omega_1 - \omega_2) - \|\omega_1 - \omega_0\|^2 - \|\omega_1 - \omega_2\|^2
\end{align}
Then, rearranging and using that $v(\omega^*) = 0$ yields that,
\begin{align}
&2\eta v(\omega_1)^T(\omega_2 - \omega^*) + 2\eta v(\omega_0)^T(\omega_1 - \omega_2)\\
&= 2\eta (v(\omega_1))^T(\omega_1 - \omega^*) + 2\eta (v(\omega_0) - v(\omega_1))^T(\omega_1 - \omega_2)\\
&= 2\eta (v(\omega_1) - v(\omega^*))^T(\omega_1 - \omega^*) + 2\eta (v(\omega_0) - v(\omega_1))^T(\omega_1 - \omega_2)\\
&\geq 2\eta \mu\|\omega_1 - \omega^*\|^2 - 2\eta \|v(\omega_0) - v(\omega_1)\| \|\omega_1 - \omega_2\|
\end{align}
where the first term is lower bounded using strong monotonicity and the second one using Cauchy-Schwarz's inequality. Using in addition the fact that $v$ is Lipschitz continuous we obtain:
\begin{align}
&2\eta v(\omega_1)^T(\omega_2 - \omega^*) + 2\eta v(\omega_0)^T(\omega_1 - \omega_2)\\
&\geq 2\eta \mu\|\omega_1 - \omega^*\|^2 - 2\eta L\|\omega_0 - \omega_1\|\|\omega_1 - \omega_2\|\\
&\geq 2\eta \mu\|\omega_1 - \omega^*\|^2 - (\eta^2 L^2\|\omega_0 - \omega_1\|^2 + \|\omega_1 - \omega_2\|^2)
\end{align}
where the last inequality comes from Young's inequality.
Using this inequality in \cref{eq: sum eq global cv proof} yields:
\begin{align}
\|\omega_2 - \omega^*\|^2 &\leq  \|\omega_0 - \omega^*\|^2  - 2\eta \mu\|\omega_1 - \omega^*\|^2 + (\eta^2 L^2 - 1)\|\omega_0 - \omega_1\|^2\,.
\end{align}
Now we lower bound $\|\omega_1 - \omega^*\|$ using $\|\omega_0 - \omega^*\|$. Indeed, from Young's inequality we obtain
\begin{align}
2\|\omega_1 - \omega^*\|^2   &\geq \|\omega_0 - \omega^*\|^2 - 2\|\omega_0 - \omega_1\|^2\,.
\end{align}
Hence, we have that,
\begin{align}
\|\omega_2 - \omega^*\|^2 &\leq (1 - \eta\mu) \|\omega_0 - \omega^*\|^2 + (\eta^2 L^2 + 2\eta \mu- 1)\|\omega_0 - \omega_1\|^2\,.
\end{align}
Note that if $\eta \leq \frac{1}{4L}$, as $\mu \leq L$, $\eta^2 L^2 + 2\eta \mu- 1 \leq -\frac{7}{16}$. Therefore, with $c = \frac{7}{16}$,
\begin{align}
\|\omega_2 - \omega^*\|^2 &\leq (1 - \eta\mu) \|\omega_0 - \omega^*\|^2 -c\|\omega_0 - \omega_1\|^2\\
&= (1 - \eta\mu) \|\omega_0 - \omega^*\|^2 -c\eta^2\|v(\omega_0)\|^2\,.
\end{align}
Finally, using $(iii)$ and \Cref{lemma: weak strong monotonicity}, we obtain:
\begin{equation}
\|\omega_2 - \omega^*\|^2 \leq (1 - \eta\mu - c\eta^2\gamma^2) \|\omega_0 - \omega^*\|^2
\end{equation}
which yields the result.
\end{proof}

\begin{prop}\label{prop: global rate proximal}
Under the assumptions of \Cref{thm: global rate eg}, the iterates of the proximal point method method $(\omega_t)_t$ with $\eta>0$ converge linearly to $\omega^*$ the unique stationary point of $v$,
\begin{equation}
\|\omega_t - \omega^*\|^2 \leq \left(1 - \frac{2\eta\mu + \eta^2\gamma^2}{1 + 2\eta\mu + \eta^2\gamma^2}\right)^t \|\omega_0 - \omega^*\|^2 \quad \forall t \geq 0\,.
\end{equation}
\end{prop}
\begin{proof}
To proof this convergence result, we upper bound the singular values of the proximal point operator $P_\eta$.
As $v$ is monotone, by \Cref{lemma: relation global quantities to spectrum}, the eigenvalues of $\nabla v$ have all non-negative real parts everywhere. As in the proof of \Cref{lemma: spectra operators}, $\omega \mapsto \omega + \eta v(\omega)$ is continuously differentiable and has a non-singular differential at every $\omega_0 \in \R^d$. By the inverse function theorem, $\omega \mapsto \omega + \eta v(\omega)$ has a continuously differentiable inverse in a neighborhood of $\omega_0$. Its inverse is exactly $P_\eta$ and it also satisfies 
\begin{equation}
\nabla P_\eta(\omega_0) = (I_d + \eta \nabla v(\omega_0))^{-1}\,.
\end{equation}
The singular values $\nabla P_\eta(\omega_0)$ are the eigenvalues of $(\nabla P_\eta(\omega_0))^T(\nabla P_\eta(\omega_0))$. The latter is equal to:
\begin{equation}
(\nabla P_\eta(\omega_0))^T(\nabla P_\eta(\omega_0)) = \left(I_d + \eta\nabla v(\omega_0) + \eta(\nabla v(\omega_0))^T + \eta^2(\nabla v(\omega_0))^T(\nabla v(\omega_0))\right)^{-1}\,.
\end{equation}
Now, let $\lambda \in \R$ be an eigenvalue of $(\nabla P_\eta(\omega_0))^T(\nabla P_\eta(\omega_0))$ and let $X \neq 0$ be one of its associated eigenvectors. As $\nabla P_\eta(\omega_0)$ is non-singular, $\lambda \neq 0$ and applying the previous equation yields:
\begin{equation}
    \lambda^{-1} X = \left(I_d + \eta\nabla v(\omega_0) + \eta(\nabla v(\omega_0))^T + \eta^2(\nabla v(\omega_0))^T(\nabla v(\omega_0))\right)X\,.
\end{equation}
Finally, multiply this equation on the left by $X^T$:
\begin{equation}
    \lambda^{-1} \|X\|^2 =  \|X\|^2 + \eta X^T(\nabla v(\omega_0) + (\nabla v(\omega_0))^T)X + \eta^2\|\nabla v(\omega_0)) X\|^2\,.
\end{equation}
Applying the first part of \Cref{lemma: relation global quantities to spectrum} yields
\begin{equation}
    \lambda^{-1} \|X\|^2 \geq  ( 1 + 2\eta\mu + \eta^2\gamma^2)\|X\|^2\,.
\end{equation}
Hence, as $X\neq 0$, we have proven that,
\begin{equation}
\sigma_{max}(\nabla v(\omega_0)) \leq ( 1 + 2\eta\mu + \eta^2\gamma^2)^{-1}\,.
\end{equation}
This implies that, for all $\omega, \omega' \in \R^d$, 
\begin{align}
    \|P_\eta(\omega) - P_\eta(\omega')\|^2 &= \left\Vert\int_{0}^{1} \nabla v(\omega' + t(\omega - \omega'))(\omega - \omega')\right\Vert^2\\
    &\leq ( 1 + 2\eta\mu + \eta^2\gamma^2)^{-1} \|\omega - \omega'\|^2\,.
\end{align}
Hence, as $P_{\eta}(\omega^*) = \omega^*$, taking $\omega' = \omega^*$ gives the desired global convergence rate.
\end{proof}

Now let us prove the result \cref{thm: global rate og} regarding Optimistic method. 
\thmglobalconvergenceoptimistic*
\begin{proof}
For the beginning of this proof we follow the proof of~\citet[Theorem 1]{gidelVariationalInequalityPerspective2018a} using their notation:
\begin{align}
  \omega_t' = \omega_t - \eta v(\omega'_{t-1})  \\
  \omega_{t+1} =\omega_t - \eta v(\omega_t') 
\end{align}
Note that, with this notation, summing the two upates steps, we recover~\eqref{equation: optimistic method}
\begin{equation}
    \omega_{t+1}' = \omega'_t - 2 \eta v(\omega_t') + \eta v(\omega_{t-1}') \,.
\end{equation}
Let us now recall~\citet[Equation 88]{gidelVariationalInequalityPerspective2018a} for a constant step-size $\eta_t = \eta$,
\begin{align}
  \|\omega_{t+1}-\omega^*\|_2^2
  &\leq  \left( 1- \eta \mu \right) \|\omega_t-\omega^*\|_2^2 \notag
         +\eta^2L^2 ( 4 \eta^2 L^2 \|\omega'_{t-1}-\omega'_{t-2}\|_2^2 - \|\omega'_{t-1}-\omega'_t\|_2^2) \\
  & \quad- (1 - 2 \eta \mu - 4  \eta^2 L^2 )\|\omega'_t-\omega_t\|_2^2 \label{eq:eq from gidel}
\end{align}
we refer the reader to the proof of~\citet[Theorem 1]{gidelVariationalInequalityPerspective2018a} for the details on how to get to this equation. Thus with $\eta \leq (4L)^{-1}$, using the update rule $\omega_t' = \omega_t - \eta v(\omega'_{t-1})$, we get, 
\begin{equation}\label{eq:inter proof optim}
    (1 - 2 \eta_t \mu - 4  \eta_t^2 L^2 )\|\omega'_t-\omega_t\|_2^2 \geq \frac{1}{4} \|\omega'_t-\omega_t\|_2^2 = \frac{\eta^2}{4} \|v(\omega'_{t-1})\|_2^2 \geq \frac{\eta^2 \gamma^2}{4} \|\omega'_{t-1}- \omega^*\|_2^2
\end{equation}
where for the last inequality we used that $\sigma_{\min}(\nabla v) \geq \gamma$ and Lemma~\ref{lemma: weak strong monotonicity}. Using Young's inequality, the update rule and the Lipchitzness of $v$, we get that,
\begin{align}
    2\|\omega'_{t-1}- \omega^*\|_2^2 
    &\geq \|\omega_{t}- \omega^*\|_2^2 - 2\|\omega'_{t-1}- \omega_t\|_2^2\\
    &= \|\omega_{t}- \omega^*\|_2^2 - 2\eta^2\|v(\omega'_{t-1})- v(\omega'_{t-2})\|_2^2\\
    &\geq  \|\omega_{t}- \omega^*\|_2^2 - 2\eta^2L^2\|\omega'_{t-1}- \omega'_{t-2}\|_2^2 \label{eq:link semi-iterates}
\end{align}
Thus combining~\eqref{eq:eq from gidel},~\eqref{eq:inter proof optim} and~\eqref{eq:link semi-iterates}, we get with a constant step-size $\eta \leq (4L)^{-1}$,
\begin{align}
\|\omega_{t+1}-\omega^*\|_2^2
    &\leq  \left( 1- \eta \mu - \frac{\eta^2\gamma^2}{8}\right) \|\omega_t-\omega^*\|_2^2 \notag
         +\eta^2L^2 \big( (4 \eta^2 L^2 + \tfrac{\eta^2\gamma^2}{4}) \|\omega'_{t-1}-\omega'_{t-2}\|_2^2 - \|\omega'_{t-1}-\omega'_t\|_2^2\big)
\end{align}
This leads to, 
\begin{align}
\|\omega_{t+1}-\omega^*\|_2^2 + \eta^2L^2 \|\omega'_{t-1}-\omega'_t\|_2^2
    &\leq  \left( 1- \eta \mu - \frac{\eta^2\gamma^2}{8}\right) \|\omega_t-\omega^*\|_2^2 
         +  \eta^2(4 L^2 + \tfrac{\gamma^2}{4}) \eta^2L^2  \|\omega'_{t-1}-\omega'_{t-2}\|_2^2 
\end{align}
In order to get the theorem statement we need a rate on $\omega_t'$. We first unroll this geometric decrease and notice that 
\begin{align}
    \|\omega'_{t}- \omega^*\|_2^2 
    &\leq 2\|\omega_{t+1}- \omega^*\|_2^2 + 2\|\omega'_{t}- \omega_{t+1}\|_2^2\\
    &= 2\|\omega_{t+1}- \omega^*\|_2^2 + 2\eta^2\|v(\omega'_{t-1})- v(\omega'_{t})\|_2^2\\
    &= 2\|\omega_{t+1}- \omega^*\|_2^2 + 2\eta^2L^2\|\omega'_{t-1}- \omega'_{t}\|_2^2 
\end{align}
to get (using the fact that $\omega'_0 = \omega'_{-1}$),
\begin{align}
    \|\omega'_{t}- \omega^*\|_2^2
     &\leq 
2\|\omega_{t+1}-\omega^*\|_2^2 + 2\eta^2L^2 \|\omega'_{t-1}-\omega'_t\|_2^2 \\
    &\leq  2\max\Big\{ 1- \eta\mu - \frac{\eta^2\gamma^2}{8L^2}, 4\eta^2L^2 + \frac{\eta^2 \gamma^2}{4}\Big\}^{t+1} \|\omega_0-\omega^*\|_2^2 \,.
\end{align}
With $\eta \leq (4L)^{-1}$ we can use the fact that $\max(\mu,\gamma) \leq L$ to get,
\begin{equation}
  1- \eta\mu - \frac{\eta^2\gamma^2}{8} \geq  1- \frac{1}{4} \Big(\frac{\mu}{L} + \frac{\gamma^2}{32L^2}\Big)\geq \frac{3}{4} - \frac{1}{32}\geq \frac{1}{4} + \frac{1}{64} \geq 4\eta^2L^2 + \frac{\eta^2 \gamma^2}{4}  
\end{equation}
Thus, 
\begin{equation}
    \max\Big\{ 1- \eta\mu - \frac{\eta^2\gamma^2}{8}, 4\eta^2L^2 + \frac{\eta^2 \gamma^2}{4}\Big\} = 
   1- \eta\mu - \frac{\eta^2\gamma^2}{8}\,, \qquad \forall \eta \leq (4L)^{-1}
\end{equation}
leading to the statement of the theorem.
Finally, for $\eta = (4L)^{-1}$
that can be simplified into,
\begin{align}
\|\omega'_{t}-\omega^*\|_2^2 
    &\leq  2\max\Big\{ 1- \frac{1}{4} \Big(\frac{\mu}{L} + \frac{\gamma^2}{32L^2}\Big), \frac{1}{4} + \frac{\gamma^2}{64L^2}\Big\}^{t+1} \|\omega_0-\omega^*\|_2^2 \,. \\
    & = 2\Big(1- \frac{1}{4} \Big(\frac{\mu}{L} + \frac{\gamma^2}{32L^2}\Big)\Big)^{t+1} \|\omega_0-\omega^*\|_2^2 
\end{align}
\end{proof}

\subsection{Proof of \cref{subsection: consensus}: consensus optimization}\label{subsection: app proof consensus}
Let us recall \cref{eq: co iterates} here,
\begin{equation}
\omega_{t+1} = \omega_t - (\alpha v(\omega_t) + \beta \nabla H(\omega_t))\,.\tag{\ref{eq: co iterates}}
\end{equation}
where $H$ is the squared norm of $v$.
We prove a more detailed version \cref{thm: consensus}.
\begin{thm}[{restate=[name=]propConsensus}]
Let $v : \R^d \rightarrow \R^d$ be continuously differentiable such that 
\begin{enumerate*}[series = tobecont, itemjoin = \quad, label=(\roman*)]
\item $v$ is $\mu$- strongly monotone for some $\mu \geq 0$,
\item $\sigma_{min}(\nabla v) \geq \gamma$ for some $\gamma > 0$
\item  $H$ is $L^2_H$ Lipschitz-smooth.
\end{enumerate*} Then, for
$$ \alpha^2 \leq \frac{1}{2}\left(\frac{\alpha \mu }{L_H} + \frac{\beta \gamma^2}{2L_H}\right)\,,$$
and $\beta \leq (2L_H)^{-1}$,
the iterates of CO defined by \cref{eq: co iterates} satisfy,
\begin{equation}
H(\omega_t) \leq \left(1 - \alpha \mu - \tfrac{1}{2}\beta \gamma^2\right) H(\omega_{0})\,.
\end{equation}
In particular, for
$$\alpha = \frac{\mu + \sqrt{\mu^2 + 2\gamma^2}}{4 L_H}\,,$$
and $\beta = (2L_H)^{-1}$,
\begin{equation}
H(\omega_t) \leq \left(1 - \frac{\mu^2}{2L_H} - \frac{\gamma^2}{2L_H^2}\left(1 + \frac{\mu}{\gamma}\right)\right) H(\omega_{0})\,.
\end{equation}
\end{thm}
\begin{proof}
As $H$ is $L_H^2$ Lipschitz smooth, we have,
\begin{equation*}
H(\omega_{t+1}) - H(\omega_t) \leq \nabla H(\omega_t)^T(\omega_{t+1} - \omega_t) + \frac{L_H^2}{2}\|\omega_{t+1} - \omega_t\|^2\,.    
\end{equation*}
Then, replacing $\omega_{t+1} - \omega_t$ by its expression and using Young's inequality,
\begin{align*}
H(\omega_{t+1}) - H(\omega_t) &\leq - \alpha \nabla H(\omega_t)^Tv(\omega_t) - \beta \|\nabla H(\omega_t)\|^2 + L_H^2\alpha^2\|v(\omega_{t})\|^2 + L_H^2\beta^2 \| \nabla H(\omega_t)\|^2\,.
\end{align*}
Note that, crucially,  $\nabla H(\omega_t) = \nabla v(\omega_t)^T v(\omega_t)$.
Using the first part  of \cref{lemma: relation global quantities to spectrum} to introduce $\mu$ and assuming $\beta \leq (2L_H^2)^{-1}$,
\begin{align*}
H(\omega_{t+1}) - H(\omega_t) &\leq - \alpha \mu \|v(\omega_t)\|^2 - \frac{\beta}{2} \|\nabla H(\omega_t)\|^2 + L_H^2\alpha^2\|v(\omega_{t})\|^2\,.
\end{align*}
Finally, using \cref{lemma: relation global quantities to spectrum} to introduce $\gamma$,
\begin{align*}
H(\omega_{t+1}) - H(\omega_t) &\leq - \alpha \mu \|v(\omega_t)\|^2 - \frac{\beta\gamma^2}{2} \|v(\omega_t)\|^2 + L_H^2\alpha^2\|v(\omega_{t})\|^2\\
&= -2\left(\alpha\mu + \frac{\beta\gamma^2}{2} - L_H^2\alpha^2\right)H(\omega_t)\,.
\end{align*}
Hence, if 
\begin{equation}\label{eq: proof consensus condition alpha}
\alpha^2 \leq \frac{1}{2}\left(\frac{\alpha \mu }{L_H^2} + \frac{\beta \gamma^2}{2L_H^2}\right)\,,
\end{equation}
then the decrease of $H$ becomes,
\begin{equation*}
H(\omega_t) \leq \left(1 - \alpha \mu - \tfrac{1}{2}\beta \gamma^2\right) H(\omega_{0})\,.
\end{equation*}
Now, note that \cref{eq: proof consensus condition alpha} is a second-order polynomial condition on $\alpha$, so we can compute the biggest $\alpha$ which satisfies this condition. This yields,
\begin{align*}
\alpha &= \frac{\frac{\mu}{2} + \sqrt{\frac{\mu^2}{4} + L_H^2 \beta \gamma^2}}{2L_H^2}\\
&= \frac{\mu + \sqrt{\mu^2 + 2\gamma^2}}{4L_H^2}\,,
\end{align*}
where in the second line we defined $\beta = (2L_H^2)^{-1}$.
Then the rate becomes,
\begin{align*}
\alpha \mu + \frac{1}{2}\beta \gamma^2 &= \frac{\mu^2}{4L_H^2} + \frac{\mu\sqrt{\mu^2 + 2\gamma^2}}{4L_H^2} + \frac{\gamma^2}{4L_H^2}\\
&\geq \frac{\mu^2}{4L_H^2} + \frac{\mu^2}{4\sqrt 2 L_H^2} + \frac{\mu\gamma}{4L_H^2} + \frac{\gamma^2}{4L_H^2}\,,
\end{align*}
where we use Young's inequality: $\sqrt{2}\sqrt{a+b} \geq \sqrt{a} + \sqrt{b}$. Noting that $\frac{1}{2}(1 + \frac{1}{\sqrt{2}}) \geq 1$ yields the result.
\end{proof}

\begin{rmk}\label{rmk: co-coercivity}
A common convergence result for the gradient method for variational inequalities problem -- see \citet{nesterovSolvingStronglyMonotone2006} for instance -- is that the iterates convergence as $O\left(\left(1 - \frac{\mu^2}{L^2}\right)^t\right)$ where $\mu$ is the monotonicty constant of $v$ and $L$ its Lipschitz constant. However, this rate is not optimal, and also not satisfying as it does not recover the convergence rate of the gradient method for strongly convex optimization. One way to remedy this situation is to use the \emph{co-coercivity} or \emph{inverse strong monotonicity} assumption:
$$\ell (v(\omega) - v(\omega'))^T(\omega - \omega') \geq \|v(\omega) - v(\omega')\|^2\quad \forall \omega, \omega'\,.$$
This yields a convergence rate of $O\left(\left(1 - \frac{\mu}{\ell}\right)^t\right)$ which can be significantly better than the former since $\ell$ can take all the values of $[L, L^2/\mu]$~\citep[\S 12.1.1]{facchineiFiniteDimensionalVariationalInequalities2003}. On one hand, if $v$ is the gradient of a convex function, $\ell = L$ and so we recover the standard rate in this case.  On the other, one can consider for example the operator $v(w)= Aw$ with $A = \begin{pmatrix}
a & -b \\ b & a 
\end{pmatrix}
$ with $a>0$ and $b \neq 0$ for which $\mu =a, L = \sqrt{a^2 + b^2}$ anf $\ell = \mu / L^2$.

\end{rmk}

\section{The \textit{p}-SCLI framework for game optimization}\label{section: p-scli game optimization}
The approach we use to prove our lower bounds comes from \citet{arjevaniLowerUpperBounds2016}. Though their whole framework was developed for convex optimization, a careful reading of their proof shows that most of their results carry on to games, at least those in their first three sections. However, we work only in the restricted setting of $1$-SCLI and so we actually rely on a very small subset of their results, more exactly two of them. 

The first one is \Cref{thm: lower bound from spectral radius} and is crucially used in the derivation of our lower bounds. We state it again for clarity.
\thmLowerBoundFromSpectralRadius*
Actually, as $F_\N : \R^d \rightarrow \R^d$ is an affine operator and $\omega^*$ is one of its fixed point, this theorem is only a reformulation of \citet[Lemma 10]{arjevaniLowerUpperBounds2016}, which is a standard result in linear algebra. So we actually do not rely on their most advanced results which were proven only for convex optimization problems.
For completeness, we state this lemma below and show how to derive \Cref{thm: lower bound from spectral radius} from it.
\begin{lemma}[{\citet[Lemma 10]{arjevaniLowerUpperBounds2016}}]
Let $A \in \R^{d\times d}$.
There exists $c > 0$, $d \geq m \geq 1$ integer and $r \in \R^d$, $r \neq 0$ such that for any $u \in \R^d$ such that $u^Tr \neq 0$, for sufficiently large $t\geq 1$ one has:
\begin{equation}
\|A^tu\| \geq ct^{m-1}\rho(A)^t\|u\|\,.
\end{equation}
\end{lemma}
\begin{proof}[Proof of \cref{thm: lower bound from spectral radius}]
$F_\N$ is affine so it can be written as $F_N(\omega) = \nabla F_\N \omega + F_\N(0)$. 

Moreover, as $v(\omega^*)= 0$, $F_\N(\omega^*) = \omega^* + \N(\nabla v)v(\omega^*) = \omega^*$. Hence, for all $\omega \in \R^d$,
\begin{equation}
F_\N(\omega) - \omega^* = F_\N(\omega) - F_\N(\omega^*) = \nabla F_\N (\omega - \omega^*)\,.
\end{equation}
Therefore, for $t\geq 0$,
\begin{equation}
\|\omega_t - w^*\| = \| (\nabla F_\N)^t (\omega - \omega^*)\|\,.
\end{equation}
Finally, apply the lemma above to $A = \nabla F_\N$. The condition $(\omega_0 - \omega^*)^Tr \neq 0$ is not satisfied only on an affine subset of dimension 1, which is of measure zero for any measure absolutely continuous w.r.t.~the Lebesgue measure. Hence for almost every $\omega_0 \in \R^d$ w.r.t. to such measure, $(\omega_0 - \omega^*)^Tr \neq 0$ and so one has, for $t\geq 1$ large enough,
\begin{align}
\|\omega_t - w^*\| &\geq ct^{m-1}\rho(\nabla F_\N)^t\|\omega_t - \omega^*\|\\
&\geq c\rho(\nabla F_\N)^t\|\omega_t - \omega^*\|\,,
\end{align}
which is the desired result.
\end{proof}

The other result we use is more anecdotal :  it is their consistency condition, which is a necessary condition for an $p$-SCLI method to converge to a stationary point of the gradient dynamics. Indeed, general 1-SCLI as defined in \citet{arjevaniLowerUpperBounds2016} are given not by one but by two mappings $\mathcal C, \N : \R^{d\times d} \rightarrow \R^{d\times d}$ and the update rule is
\begin{equation}\label{eq: def general pscli}
F_{\mathcal \N}(\omega) = \mathcal C(\nabla v)w + \N(\nabla v)v(0)\quad \forall \omega \in \R^d\,.
\end{equation}
However, they show in \citet[Thm. 5]{arjevaniLowerUpperBounds2016} that, for a method to converge to a stationary point of $v$, at least for convex problems, that is to say symmetric positive semi-definite $\nabla v$, $\mathcal C$ and $\mathcal N$ need to satisfy:
\begin{equation}
I_d - \mathcal C(\nabla v) = -\mathcal N(\nabla v)\nabla v\,.
\end{equation}
If $\mathcal C$ and $\mathcal N$ are polynomials, this equality for all symmetric positive semi-definite $\nabla v$ implies the equality on all matrices. Injecting this result in \cref{eq: def general pscli} yields the definition of 1-SCLI we used.
\section{Proofs of lower bounds}\label{section: app lower bounds}\label{subsection: app lower bounds n-extra}

The class of methods we consider, that is to say the methods whose coefficient mappings $\N$ are any polynomial of degree at most $k - 1$, is very general. It includes: 
\begin{itemize}[label=-]
    \item the $k'$-extrapolation methods $F_{k', \eta}$ for $k' \leq k$ as defined by \cref{eq: k-extrapolation methods}.
    \item extrapolation methods with different step sizes for each extrapolation:
    \begin{equation}
        \omega \longmapsto \varphi_{\eta_1,\omega} \circ \varphi_{\eta_2,\omega} \circ \dots \circ \varphi_{\eta_{k},\omega}(\omega)\,,
    \end{equation}
    \item cyclic Richardson iterations \citep{opferRichardsonIterationNonsymmetric1984}: methods whose update is composed of successive gradient steps with possibly different step sizes for each
    \begin{equation}
         \omega \longmapsto F_{\eta_1} \circ F_{\eta_2} \circ \dots \circ F_{\eta_{k}}(\omega)\,,
    \end{equation}
\end{itemize}
and any combination of these with at most $k$ composed gradient evaluations.

The lemma below shows how $k$-extrapolation algorithms fit into the definition of $1$-SCLI:
\begin{lemma}[{restate=[name=]lemmaInverseMatrixExtrapolationMethod}]\label{lemma: inverse matrix extrapolation method}
For a $k$-extrapolation method, $\N(\nabla v) = -\eta \sum_{j = 0}^{k-1} (- \eta \nabla v)^k$.
\end{lemma}
 \begin{proof}
This result is a direct consequence of \Cref{lemma: jacobians of extrapolation operators}. For $\omega \in \R^d$, one gets, by the chain rule, 
\begin{align}
\nabla F_{\eta, k}(\omega) &= \nabla_z \varphi_{\omega}^k(\omega) + \nabla_\omega \varphi_{\omega}^k(\omega)\\
&= (-\eta \nabla v)^k +  \sum_{j=0}^{k-1} (-\eta \nabla v)^{j}\\
&= \sum_{j=0}^{k} (-\eta \nabla v)^{j}\,.
\end{align}
as $\nabla v$ is constant. Hence, as expected, $F_{\eta, k}$ is linear so write that, for all $\omega \in \R^d$,
\begin{equation}
F_{\eta, k}(\omega) = \nabla F_{\eta, k}\omega + b\,.
\end{equation}
If $v$ has a stationary point $\omega^*$, evaluating at $\omega^*$ yields
\begin{equation}
\omega^* =  \sum_{j=0}^{k} (-\eta \nabla v)^{j}\omega^* + b\,.
\end{equation}
Using that $v(\omega^*) = 0$ and so $(\nabla v) \omega^* = -v(0)$, one gets that
\begin{equation}
b = -\eta\sum_{j=1}^{k} (-\eta \nabla v)^{j-1}v(0)\,,
\end{equation}
and so 
\begin{equation}
F_{\eta, k}(\omega) = \omega -\eta\sum_{j=1}^{k} (-\eta \nabla v)^{j-1}v(\omega)\,,
\end{equation}
which yields the result for affine vector fields with a stationary point. In particular it holds for vector fields such that $\nabla v$ is non-singular. As the previous equality is continuous in $\nabla v$, by density of non-singular matrices, the result holds for all affine vector fields.

 \end{proof}
\thmLowerBoundConvex*
To ease the presentation of the proof of the theorem, we rely on several lemmas. We first prove $(i)$ and $(ii)$ will follow as a consequence.

In the following, we denote by $\R_{k-1}[X]$ the set of real polynomials of degree at most $k-1$.
\begin{lemma}\label{lemma: spectral radius k extrapolation}
For, $v \in \mathcal{V}_d$, 
\begin{align}\label{eq: spectral radius k extrapolation}
\min_{N \in \R_{k-1}[X]} \half\rho(F_{\N})^2 &= \min_{a_0,\dots,a_{k-1} \in \R} \max_{\lambda \in \Sp \nabla v} \half|1+\sum_{l=0}^{k-1} a_l\lambda^{l+1}|^2\,.
\end{align}
\end{lemma}
\begin{proof}
Recall the definition of $F_{\N}$, which is affine by assumption, 
\begin{equation}
\forall \omega \in \R^d,\ F_{\N}(\omega) = w + \N(\nabla v)v(\omega)\,.
\end{equation}
Then $\nabla F_{\N} = I_d + \N(\nabla v)\nabla v$. As $\N$is a polynomial, by the spectral mapping theorem (\Cref{thm: spectral mapping theorem}), 
\begin{equation}
\Sp\nabla F_{\N} = \{1 +    \N(\lambda)\lambda\ |\ \lambda \in \Sp \nabla v\}\,,
\end{equation}
which yields the result.
\end{proof}
\begin{lemma}\label{lemma: dual problem}
Assume that $\Sp \nabla v = \{\lambda_1,\dots,\lambda_m\} \subset \R$. Then \cref{eq: spectral radius k extrapolation} can be lower bounded by the value of the following problem:
\begin{equation}
\begin{aligned}
\max &\sum_{j=1}^m \nu_j(\xi_j - \frac{1}{2}\xi_j^2)\\
\text{s.t.}\ \  &\nu_j \geq 0,\ \ \xi_j \in \R,\ \forall 1\leq j \leq m\\
&\sum_{j=1}^m \nu_j \xi_j \lambda_j^l = 0,\ \forall 1\leq l\leq k\\ 
&\sum_{j=1}^m \nu_j = 1\\
\end{aligned}
\end{equation}
\end{lemma}
\begin{proof}
The right-hand side of \cref{eq: spectral radius k extrapolation} can be written as a constrained optimization problem as follows:
\begin{equation}
\begin{aligned}
&\min_{t, a_0,\dots,a_{k-1},z_1,\dots,z_m \in \R} t\\
&\text{s.t.}\ \ t \geq \frac{1}{2} z_j^2,\ \forall\ 1\leq j \leq m\\
&z_j = 1+\sum_{l=0}^{k-1} a_l\lambda_j^{l+1},\ \forall\ 1\leq j \leq m\,.\\
\end{aligned}    
\end{equation}
By weak duality, see \citet{boydConvexOptimization2004a} for instance, we can lower bound the value of this problem by the value of its dual. So let us write the Lagrangian of this problem:
\begin{equation}
\begin{aligned}
&\mathcal{L}(t, a_0,\dots,a_{k-1},z_1,\dots,z_m, \nu_1,\dots,\nu_m, \chi_1,\dots,\chi_m)\\
&=t + \sum_{j=0}^m \nu_j(\half z_j^2 - t) + \chi_j(1+\sum_{l=0}^{k-1} a_l\lambda_j^{l+1} - z_j)\,.\\
\end{aligned}
\end{equation}
The Lagrangian is convex and quadratic so its minimum with respect to $t, a_0,\dots,a_{k-1},z_1,\dots,z_m$ is characterized by the first order condition. Moreover, if there is no solution to the first order condition, its minimum is $-\infty$ (see for instance \citet[Example 4.5]{boydConvexOptimization2004a}).

One has that, for any $1 \leq j \leq m$ and $0 \leq l \leq k-1$,
\begin{align}
&\partial_t \mathcal{L} = 1 - \sum_{j=0}^m \nu_j\\
&\partial_{a_l} \mathcal{L} = \sum_{j=0}^m \chi_j \lambda_j^{l+1}\\
&\partial_{z_j} \mathcal{L} = \nu_j z_j - \chi_j\,.
\end{align}
Setting these quantities to zero yields the following dual problem:
\begin{equation}
\begin{aligned}
\max &\sum_{j=1,\ \nu_j \neq 0}^m \chi_j - \frac{1}{2\nu_j}\chi_j^2\\
\text{s.t.}\ \  &\nu_j \geq 0,\ \ \chi_j \in \R,\ \forall 1\leq j \leq m\\
&\sum_{j=1}^m \chi_j \lambda_j^l = 0,\ \forall 1\leq l\leq k\\
&\nu_j = 0 \implies \chi_j = 0\\
&\sum_{j=1}^m \nu_j = 1\,
\end{aligned}
\end{equation}
Taking $\nu_j\xi_j = \chi_j$ yields the result:
\begin{equation}
\begin{aligned}
\max &\sum_{j=1}^m \nu_j(\xi_j - \frac{1}{2}\xi_j^2)\\
\text{s.t.}\ \  &\nu_j \geq 0,\ \ \xi_j \in \R,\ \forall 1\leq j \leq m\\
&\sum_{j=1}^m \nu_j \xi_j \lambda_j^l = 0,\ \forall 1\leq l\leq k\\ 
&\sum_{j=1}^m \nu_j = 1\,.
\end{aligned}
\end{equation}
\end{proof}
The next lemma concerns Vandermonde matrices and Lagrange polynomials.
\begin{lemma}
Let $\lambda_1,\dots,\lambda_d$ be distinct reals. Denote the Vandermonde matrix by \begin{equation}
V(\lambda_1,\dots,\lambda_d) = \begin{pmatrix}
1 & \lambda_1 & \lambda_1^2 & \dots & \lambda_1^{d-1}\\
1 & \lambda_2 & \lambda_2^2 & \dots & \lambda_2^{d-1}\\
\vdots & \vdots & \vdots & & \vdots\\
1 & \lambda_1 & \lambda_1^2 & \dots & \lambda_1^{d-1}\\
\end{pmatrix}\,.
\end{equation}
Then 
\begin{equation}
V(\lambda_1,\dots,\lambda_d)^{-1} = \begin{pmatrix}
L_1^{(0)} & L_2^{(0)} & \dots & L_{d}^{(0)}\\
L_1^{(1)} & L_2^{(1)} & \dots & L_{d}^{(1)}\\
\vdots    & \vdots    &       & \vdots\\
L_1^{(d-1)} & L_2^{(d-1)} & \dots & L_{d}^{(d-1)}\\
\end{pmatrix}
\end{equation}
where $L_1, L_2, \dots, L_d$ are the Lagrange interpolation polynomials associated to $\lambda_1,\dots, \lambda_d$ and $L_j = \sum_{l=0}^{d-1} L_j^{(l)}X^l$ for $1 \leq j \leq d$.
\end{lemma}
A proof of this result can be found at \citet[Theorem 3.1]{atkinsonIntroductionNumericalAnalysis1989}.

The next lemma is the last one before we finally prove the theorem. Recall that in \Cref{thm: lower bound convex} we assume that $k+1 \leq d$.
\begin{lemma}\label{lemma: Lagrange lower bound}
Assume that $\Sp \nabla v = \{\lambda_1,\dots,\lambda_{k+1}\}$ where $\lambda_1,\dots,\lambda_{k+1}$ are distinct non-zero reals. Then the problem of \cref{eq: spectral radius k extrapolation} is lower bounded by 
\begin{equation}\label{eq: Lagrange lower bound}
\half \left(\frac{1 - \sum_{j=1}^k\frac{\lambda_{k+1}}{\lambda_j}L_j(\lambda_{k+1})}{1+\sum_{j=1}^k|\frac{\lambda_{k+1}}{\lambda_j}L_j(\lambda_{k+1})|}\right)^2\,,
\end{equation}
where $L_1,\dots,L_k$ are the Lagrange interpolation polynomials associated to $\lambda_1,\dots,\lambda_{k}$.
\end{lemma}
\begin{proof}
To prove this lemma, we start from the result of \Cref{lemma: dual problem} and we provide feasible $(\nu_j)_j$ and $(\xi_j)_j$. First, any feasible $(\nu_j)_j$ and $(\xi_j)_j$ must satisfy the $k$ constraints involving the powers of the eigenvalues, which can be rewritten as:
\begin{equation}
V(\lambda_1,\dots,\lambda_k)^T \begin{pmatrix}
\nu_1\xi_1\lambda_1\\
\nu_2\xi_2\lambda_2\\
\vdots\\
\nu_k\xi_k\lambda_k\\
\end{pmatrix}
= -\nu_{k+1}\xi_{k+1}\begin{pmatrix}
\lambda_{k+1}\\
\lambda_{k+1}^2\\
\dots\\
\lambda_{k+1}^{k}\\
\end{pmatrix}\,.
\end{equation}
Using the previous lemma yields, for $1 \leq j \leq k$,
\begin{align}
\nu_j \xi_j &=-\nu_{k+1}\xi_{k+1}\frac{1}{\lambda_j}\begin{pmatrix}
L_j^{(0)} & L_j^{(1)} & \dots & L_j^{(k-1)}\\
\end{pmatrix}\begin{pmatrix}
\lambda_{k+1}\\
\lambda_{k+1}^2\\
\dots\\
\lambda_{k+1}^{k}\\
\end{pmatrix}\\
&= -\nu_{k+1}\xi_{k+1}\frac{\lambda_{k+1}}{\lambda_j}L_j(\lambda_{k+1})\,.
\end{align}
Hence the problem can be rewritten only in terms of the $(\nu_j)_j$ and $\xi_{k+1}$. Let $c_j = \frac{\lambda_{k+1}}{\lambda_j}L_j(\lambda_{k+1})$. The objective becomes:
\begin{equation}
\sum_{j=1}^m \nu_j\left(\xi_j - \frac{1}{2}\xi_j^2\right) = \nu_{k+1}\xi_{k+1}\left(1 - \sum_{j=1}^k c_j\right) - \half\nu_{k+1}\xi_{k+1}^2\left(1 + \sum_{j=0}^k \frac{\nu_{k+1}}{\nu_j}c_j^2\right)\,.
\end{equation}
Choosing $\xi_{k+1} = \frac{1 - \sum_{j=1}^k c_j}{1 + \sum_{j=0}^k \frac{\nu_{k+1}}{\nu_j}c_j^2}$ to maximize this quadratic yields:
\begin{equation}
\sum_{j=1}^m \nu_j\left(\xi_j - \frac{1}{2}\xi_j^2\right) = \half \nu_{k+1} \frac{\left(1 - \sum_{j=1}^k c_j\right)^2}{1 + \sum_{j=0}^k \frac{\nu_{k+1}}{\nu_j}c_j^2}\,.
\end{equation}
Finally take $\nu_j = \frac{|c_j|}{1 + \sum_{j=1}^k |c_j|}$ for $j \leq k$ and $\nu_{k+1} = \frac{1}{1 + \sum_{j=1}^k |c_j|}$which satisfy the hypotheses of the problem of \Cref{lemma: dual problem}. With the feasible $(\nu_j)_j$ and $(\xi_j)_j$ defined this way, the value of the objective is
\begin{equation}
\half \left(\frac{1 - \sum_{j=1}^kc_j}{1+\sum_{j=1}^k|c_j|}\right)^2\,,
\end{equation}
which is the desired result.
\end{proof}
We finally prove $(i)$ of \Cref{thm: lower bound convex}.
\begin{proof}[Proof of $(i)$ of \Cref{thm: lower bound convex}]
To prove the theorem, we build on the result of \Cref{lemma: Lagrange lower bound}. We have to choose $\lambda_1,\dots,\lambda_{k+1} \in [\mu, L]$ positive distinct such that \cref{eq: Lagrange lower bound} is big. One could try to distribute the eigenvalues uniformly across the interval but this leads to a lower bound which decreases exponentially in $k$. To make things a bit better, we use Chebyshev points of the second kind studied by \citet{salzerLagrangianInterpolationChebyshev1971}. However we will actually refer to the more recent presentation of \citet{berrutBarycentricLagrangeInterpolation2004}.

For now, assume that $k$ is even and so $k \geq 4$. We will only use that $d-1 \geq k$ (and not that $d-2 \geq k$).
Define, for $1 \leq j \leq k$, $\lambda_j = \frac{\mu + L}{2} - \frac{L - \mu}{2}\cos{\frac{j - 1}{k - 1}\pi}$. Using the barycentric formula of \citet[Eq.~4.2]{berrutBarycentricLagrangeInterpolation2004}, the polynomial which interpolates $f_1,\dots,f_k$ at the points $\lambda_1,\dots,\lambda_k$ can be written as:
\begin{equation}
P(X) = \frac{\sum_{j=1}^k \frac{w_j}{X - \lambda_j}f_j}{\sum_{j=1}^k \frac{w_j}{X - \lambda_j}}\,,
\end{equation}
where 
\begin{equation}
w_j = \begin{cases}
(-1)^{j-1} \quad&\text{if}\quad 2 \leq j \leq k -1\\
\half(-1)^{j-1} \quad&\text{if}\quad j \in\{1, k\}\,.\\
\end{cases}    
\end{equation}
Define $Z(X) = \sum_{j=1}^k \frac{w_j}{X - \lambda_j}$.

Now, $\sum_{j=1}^k\frac{\lambda_{k+1}}{\lambda_j}L_j(\lambda_{k+1})$ can be seen as the polynomial interpolating $\frac{\lambda_{k+1}}{\lambda_1},\dots,\frac{\lambda_{k+1}}{\lambda_k}$ at the points $\lambda_1,\dots,\lambda_j$ evaluated at $\lambda_{k+1}$. Hence, using the barycentric formula,
\begin{equation}
\sum_{j=1}^k\frac{\lambda_{k+1}}{\lambda_j}L_j(\lambda_{k+1}) =  \frac{1}{Z(\lambda_{k+1})} \sum_{j=1}^k \frac{w_j}{\lambda_{k+1} - \lambda_j}\frac{\lambda_{k+1}}{\lambda_j}\,.
\end{equation}
Similarly, $\sum_{j=1}^k|\frac{\lambda_{k+1}}{\lambda_j}L_j(\lambda_{k+1})|$ can be seen as the polynomial interpolating $|\frac{\lambda_{k+1}}{\lambda_1}|\sign(L_1(\lambda_{k+1})),\dots,|\frac{\lambda_{k+1}}{\lambda_k}|\sign(L_k(\lambda_{k+1}))$ at the points $\lambda_1,\dots,\lambda_j$ evaluated at $\lambda_{k+1}$. However, from \citet[Section 3]{berrutBarycentricLagrangeInterpolation2004}, \begin{equation}
L_j(\lambda_{k+1}) = \left(\prod_{j = 1}^k(\lambda_{k+1} - \lambda_j)\right)\frac{w_j}{\lambda_{k+1} - \lambda_j}\,,
\end{equation}
and by \citet[Eq.~4.1]{berrutBarycentricLagrangeInterpolation2004},
\begin{equation}\label{eq: sign Z}
1 = \left(\prod_{j = 1}^k(\lambda_{k+1} - \lambda_j)\right)Z(\lambda_{k+1})\,.
\end{equation}
Hence 
\begin{equation}
 \sign(L_j(\lambda_{k+1})) = \sign Z(\lambda_{k+1}) \sign\left(\frac{w_j}{\lambda_{k+1} - \lambda_j}\right)\,.
\end{equation}
Therefore, using the barycentric formula again,
\begin{equation}
\sum_{j=1}^k\frac{\lambda_{k+1}}{\lambda_j}|L_j(\lambda_{k+1})| =  \frac{1}{|Z(\lambda_{k+1})|} \sum_{j=1}^k \left|\frac{w_j}{\lambda_{k+1} - \lambda_j}\right|\frac{\lambda_{k+1}}{\lambda_j}\,.
\end{equation}
Hence, \cref{eq: Lagrange lower bound} becomes:
\begin{align}
&\half \left(\frac{1 - \sum_{j=1}^k\frac{\lambda_{k+1}}{\lambda_j}L_j(\lambda_{k+1})}{1+\sum_{j=1}^k|\frac{\lambda_{k+1}}{\lambda_j}L_j(\lambda_{k+1})|}\right)^2\\
&= \half \left(\frac{1 - \frac{1}{Z(\lambda_{k+1})} \sum_{j=1}^k \frac{w_j}{\lambda_{k+1} - \lambda_j}\frac{\lambda_{k+1}}{\lambda_j}}{1 + \frac{1}{|Z(\lambda_{k+1})|} \sum_{j=1}^k \left|\frac{w_j}{\lambda_{k+1} - \lambda_j}\right|\frac{\lambda_{k+1}}{\lambda_j}}\right)^2\\
&= \half \left(1 - \frac{\frac{1}{|Z(\lambda_{k+1})|} \sum_{j=1}^k \left|\frac{w_j}{\lambda_{k+1} - \lambda_j}\right|\frac{\lambda_{k+1}}{\lambda_j}\left(1 + \sign Z(\lambda_{k+1}) \sign \left( \frac{w_j}{\lambda_{k+1} - \lambda_j}\right)\right)}{1 + \frac{1}{|Z(\lambda_{k+1})|} \sum_{j=1}^k \left|\frac{w_j}{\lambda_{k+1} - \lambda_j}\right|\frac{\lambda_{k+1}}{\lambda_j}}\right)^2\,.\label{eq: previous result LB}
\end{align}
Now take any $\lambda_{k+1}$ such that $\lambda_1 < \lambda_{k+1} < \lambda_2$. Then, from \cref{eq: sign Z}, $\sign Z(\lambda_{k+1}) = (-1)^{k+1} = -1$ as we assume that $k$ is even. By definition of the coefficients $w_j$, $\sign \frac{w_1}{\lambda_{k+1} - \lambda_1} = +1$. Hence $1 + \sign Z(\lambda_{k+1}) \sign \frac{w_1}{\lambda_{k+1} - \lambda_1} = 0$. Similarly, $\sign \frac{w_2}{\lambda_{k+1} - \lambda_2} = +1$ and so $1 + \sign Z(\lambda_{k+1}) \sign \frac{w_2}{\lambda_{k+1} - \lambda_2} = 0$ too\footnote{We could do without this, but it is free and gives slightly better constants.}.

As the quantity inside the parentheses of \cref{eq: previous result LB} is non-negative, we can focus on lower bounding it. Using the considerations on signs we get:
\begin{align}
&\frac{\frac{1}{|Z(\lambda_{k+1})|} \sum_{j=1}^k \left|\frac{w_j}{\lambda_{k+1} - \lambda_j}\right|\frac{\lambda_{k+1}}{\lambda_j}\left(1 + \sign Z(\lambda_{k+1}) \sign \left( \frac{w_j}{\lambda_{k+1} - \lambda_j}\right)\right)}{1 + \frac{1}{|Z(\lambda_{k+1})|} \sum_{j=1}^k \left|\frac{w_j}{\lambda_{k+1} - \lambda_j}\right|\frac{\lambda_{k+1}}{\lambda_j}}\\
&=\frac{\frac{1}{|Z(\lambda_{k+1})|} \sum_{j=3}^k \left|\frac{w_j}{\lambda_{k+1} - \lambda_j}\right|\frac{\lambda_{k+1}}{\lambda_j}\left(1 + \sign Z(\lambda_{k+1}) \sign \left( \frac{w_j}{\lambda_{k+1} - \lambda_j}\right)\right)}{1 + \frac{1}{|Z(\lambda_{k+1})|} \sum_{j=1}^k \left|\frac{w_j}{\lambda_{k+1} - \lambda_j}\right|\frac{\lambda_{k+1}}{\lambda_j}}\\
&\leq 2 \frac{\sum_{j=3}^k \left|\frac{1}{\lambda_{k+1} - \lambda_j}\right|\frac{\lambda_{k+1}}{\lambda_j}}{|Z(\lambda_{k+1})| + \sum_{j=1}^k \left|\frac{w_j}{\lambda_{k+1} - \lambda_j}\right|\frac{\lambda_{k+1}}{\lambda_j}}\\
&\leq 2 \frac{\sum_{j=3}^k \left|\frac{1}{\lambda_{k+1} - \lambda_j}\right|\frac{\lambda_{k+1}}{\lambda_j}}{ \half\left|\frac{1}{\lambda_{k+1} - \lambda_1}\right|\frac{\lambda_{k+1}}{\lambda_1}}\\
&\leq 2 \frac{(k-2) \left|\frac{1}{\lambda_{k+1} - \lambda_3}\right|\frac{\lambda_{k+1}}{\lambda_3}}{ \half\left|\frac{1}{\lambda_{k+1} - \lambda_1}\right|\frac{\lambda_{k+1}}{\lambda_1}}\\
\end{align}
where we used that, for $j\geq 3$, $\left|\frac{1}{\lambda_{k+1} - \lambda_j}\right|\frac{\lambda_{k+1}}{\lambda_j} \leq \left|\frac{1}{\lambda_{k+1} - \lambda_3}\right|\frac{\lambda_{k+1}}{\lambda_3}$ as $\lambda_1 < \lambda_{k+1} < \lambda_2 < \lambda_3 < \dots < \lambda_k$.
Now, recalling that $\lambda_1 = \mu$, and using that $\lambda_1 < \lambda_{k+1} < \lambda_2 < \lambda_3$ for the inequality,
\begin{align}
2 \frac{(k-2) \left|\frac{1}{\lambda_{k+1} - \lambda_3}\right|\frac{\lambda_{k+1}}{\lambda_3}}{ \half\left|\frac{1}{\lambda_{k+1} - \lambda_1}\right|\frac{\lambda_{k+1}}{\lambda_1}} &= 4(k-2)\frac{\mu}{\lambda_3}\frac{|\lambda_{k+1} - \lambda_1|}{|\lambda_{k+1} - \lambda_3|}\\
&\leq 4(k-2)\frac{\mu}{\lambda_3}\frac{|\lambda_{2} - \lambda_1|}{|\lambda_{2} - \lambda_3|}\\
&= 4(k-2)\frac{\mu}{\half L (1 - \cos\frac{2\pi}{k-1}) + \half \mu (1 + \cos\frac{2\pi}{k-1})}\frac{|\cos\frac{\pi}{k-1} - 1|}{|\cos\frac{\pi}{k-1} - \cos\frac{2\pi}{k-1}|}\\
&\leq 8(k-2)\frac{\mu}{ L (1 - \cos\frac{\pi}{k-1})}\frac{|\cos\frac{\pi}{k-1} - 1|}{|\cos\frac{\pi}{k-1} - \cos\frac{2\pi}{k-1}|}\\
&= 8(k-2)\frac{\mu}{L}\frac{1}{|\cos\frac{\pi}{k-1} - \cos\frac{2\pi}{k-1}|}
\end{align}
by definition of the interpolation points.
Now, for $k \geq 4$, the sinus is non-negative on $[\frac{\pi}{k-1}, \frac{2\pi}{k-1}]$ and reaches its minimum at $\frac{\pi}{k-1}$. Hence,
\begin{align}
\left|\cos\frac{\pi}{k-1} - \cos\frac{2\pi}{k-1}\right| &= \left|\int_{\pi/(k-1)}^{2\pi/(k-1)} \sin t dt \right|\\
&= \int_{\pi/(k-1)}^{2\pi/(k-1)} \sin t dt\\
&\geq \frac{\pi}{k-1}\sin\frac{\pi}{k-1}\\
&\geq 2\frac{\pi}{(k-1)^2}
\end{align}
as $0 \geq \frac{\pi}{k-1} \geq \frac{\pi}{2}$.
Putting everything together yields,
\begin{align}
\half \left(\frac{1 - \sum_{j=1}^k\frac{\lambda_{k+1}}{\lambda_j}L_j(\lambda_{k+1})}{1+\sum_{j=1}^k|\frac{\lambda_{k+1}}{\lambda_j}L_j(\lambda_{k+1})|}\right)^2 &\geq \half\left(1 - \frac{4(k-1)^2(k-2)}{\pi}\frac{\mu}{L}\right)^2\\
&\geq \half\left(1 - \frac{4(k-1)^3)}{\pi}\frac{\mu}{L}\right)^2\,,
\end{align}
which yields the desired result by the definition of the problem of \cref{eq: spectral radius k extrapolation}.

The lower bound holds for any $v$ such that $\Sp \nabla v = \{\lambda_1,\dots,\lambda_{k+1}\}$. As $\{\lambda_1,\dots,\lambda_{k+1}\} \subset [\mu, L]$, one can choose $v$ of the form $v = \nabla f$ where $f:\R^d \rightarrow \R$ is a $\mu$-strongly convex and $L$-smooth quadratic function with $\Sp \nabla^2 f =  \{\lambda_1,\dots,\lambda_{k+1}\}$.

Now, we tackle the case $k$ odd, with $k\geq3$ and $d-1 \geq k+1$. Note that if $\mathcal{N}$ is a real polynomial of degree at most $k - 1$, it is also a polynomial of degree at most $(k + 1) - 1$. Applying the result above yields that there exists $v \in \mathcal{V}_d$ with the desired properties such that,
\begin{equation}
\rho(F_{\N}) \geq 1 - \frac{k^3}{2\pi} \frac{\gamma^2}{L^2}\,.
\end{equation}
Hence, $(i)$ holds for any $d -2\geq k \geq 3$.
\end{proof}
Then, $(ii)$ is essentially a corollary of $(i)$.
\begin{proof}[Proof of $(ii)$ of \Cref{thm: lower bound convex}]
For a square zero-sum two player game, the Jacobian of the vector field can be written as,
\begin{equation}
\nabla v = \begin{pmatrix}
0_m & A\\
-A^T & 0_m
\end{pmatrix}
\end{equation} where $A \in \R^{m\times m}$. By \Cref{lemma: characterization of the spectrum}, \begin{equation}
\Sp \nabla v = \{i\sqrt{\lambda}\ |\ \lambda \in \Sp AA^T\} \cup \{-i\sqrt{\lambda}\ |\ \lambda \in \Sp AA^T\}\,.
\end{equation}
Using \Cref{lemma: spectral radius k extrapolation}, one gets that:
\begin{align}
\min_{N \in \R_{k-1}[X]} \half\rho(F_{\N})^2 &= \min_{a_0,\dots,a_{k-1} \in \R} \max_{\lambda \in \Sp AA^T}\half \max\left(\left|1+\sum_{l=0}^{k-1} a_l(i\sqrt{\lambda})^{l+1}\right|^2, \left|1+\sum_{l=0}^{k-1} a_l(-i\sqrt{\lambda})^{l+1}\right|^2\right)\\
&\geq \min_{a_0,\dots,a_{k-1} \in \R} \max_{\lambda \in \Sp AA^T} \half\left|1+\sum_{l=0}^{k-1} a_l(i\sqrt{\lambda})^{l+1}\right|^2\\
&\geq \min_{a_0,\dots,a_{k-1} \in \R} \max_{\lambda \in \Sp AA^T} \half\left(\Re\left(1+\sum_{l=0}^{k-1} a_l(i\sqrt{\lambda})^{l+1}\right)\right)^2\\
&= \min_{a_0,\dots,a_{k-1} \in \R} \max_{\lambda \in \Sp AA^T} \half\left(1+\sum_{l=1}^{\lfloor k/2\rfloor} a_{2l-1}(-1)^l\lambda^{l}\right)^2\\
&= \min_{a_0,\dots,a_{\lfloor k/2\rfloor -1} \in \R} \max_{\lambda \in \Sp AA^T} \half\left(1+\sum_{l=1}^{\lfloor k/2\rfloor} a_{l-1}\lambda^{l}\right)^2\,.
\end{align}
Using \Cref{lemma: spectral radius k extrapolation} again, \begin{equation}
 \min_{a_0,\dots,a_{\lfloor k/2\rfloor -1} \in \R} \max_{\lambda \in \Sp AA^T} \half\left(1+\sum_{l=1}^{\lfloor k/2 \rfloor} a_{l-1}\lambda^{l}\right)^2 = \min_{N \in \R_{\lfloor k/2\rfloor -1}[X]} \half\rho(\tilde F_{\N})^2\,.
\end{equation}
where $\tilde F_{\N}$ is the 1-SCLI operator of $\N$, as defined by \cref{eq: 1-SCLI rule} applied to the vector field $\omega \mapsto AA^T \omega$. Let $S \in \R^{m \times m}$ be a symmetric positive definite matrix given by $(i)$ of this theorem applied with $(\mu, L) = (\gamma^2, L^2)$ and $\left\lfloor\frac{k}{2}\right\rfloor$ instead of $k$ and so such that $\Sp S \subset [\gamma^2, L^2]$. Now choose $A \in \R^{m \times m}$ such that $A^TA = S$, for instance by taking a square root of $S$ (see \citet[Chapter 10]{laxLinearAlgebraIts2007}). Then, 
\begin{equation}
    \min_{N \in \R_{\lfloor k/2 \rfloor-1}[X]} \half\rho(\tilde F_{\N})^2 \geq \half\left(1 - \frac{k^3}{2\pi} \frac{\mu}{L}\right)\,.
\end{equation}
Moreover, by computing $\nabla v^T\nabla v$ and using that $\Sp AA^T = \Sp A^TA$, one gets that $\min_{\lambda \in \Sp \nabla v} |\lambda| = \sigma_{min}(\nabla v) = \sigma_{min}(A) \geq \gamma$ and $\sigma_{max}(\nabla v) = \sigma_{max}(A) \leq L$.
\end{proof}
\begin{rmk}
Interestingly, the examples we end up using have a spectrum similar to the one of the matrix Nesterov uses in the proofs of his lower bounds in \citet{nesterovIntroductoryLecturesConvex2004}. The choice of the spectrum of the Jacobian of the vector field was indeed the choice of interpolation points. Following \citet{salzerLagrangianInterpolationChebyshev1971, berrutBarycentricLagrangeInterpolation2004} we used points distributed across the interval as a cosinus as it minimizes oscillations near the edge of the interval. Therefore, this links the hardness Nesterov's examples to the well-conditioning of families of interpolation points.
\end{rmk}

\section{Handling singularity}\label{section: app handling singularity}
The following theorem is a way to use spectral techniques to obtain geometric convergence rates even if the Jacobian of the vector field at the stationary point is singular. We only need to ensure that the vector field is locally null along these directions of singularity.

In this subsection, for $A \in \R^{m \times p}$, $\Ker A = \{x \in \R^p\, |\, Ax = 0\}$ denotes the kernel (or the nullspace) of $A$.

The following theorem is actually a combination of the proof of \citet[Thm.~A.4]{nagarajanGradientDescentGAN2017}, which only proves asymptotic statibility in continuous time with no concern for the rate, and the classic \cref{thm: local convergence}.
\begin{thm}\label{thm: local convergence lemma with singularity}
Consider $h : \R^m \times \R^p \rightarrow \R^m \times \R^p$ twice continuously  differentiable vector field and write $h(\theta, \varphi) = (h_\theta(\theta, \varphi), h_\varphi(\theta, \varphi))$. Assume that $(0, 0)$ is a stationary point, i.e. $h(0, 0) = (0, 0)$ and that there exists $\tau > 0$ such that,
\begin{equation}
\forall \varphi \in \R^p \cap B(0, \tau),\quad  h(0, \varphi) = (0, 0)\,.    
\end{equation}
Let $\rho^* = \rho(\nabla_\theta (\Id + h_\theta)(0, 0))$ and define the iterates $(\theta_t, \varphi_t)_t$ by
\begin{equation}
(\theta_{t+1}, \varphi_{t+1}) = (\theta_t, \varphi_t) + h(\theta_t, \varphi_t)\,. 
\end{equation}
Then, if $\rho^* < 1$, for all $\epsilon > 0$, there exists a neighborhood of $(0, 0)$ such that for any initial point in this neighborhood, the distance of the iterates $(\theta_t, \varphi_t)_t$ to a stationary point of $h$ decreases as 
$\bigO((\rho^*+\epsilon)^t)$\,. If $v$ is linear, this is satisfied with the whole space as a neighborhood for all $\epsilon > 0$.
\end{thm}
The following proof is inspired from the ones of  \citet[Thm.~4]{nagarajanGradientDescentGAN2017} and \citet[Thm.~1]{gidelNegativeMomentumImproved2018b}.
\newcommand{\J}{\textbf{J}}
\begin{proof}
Let $\J = \nabla_\theta h(0, 0) \in \R^{(m+p) \times m}$, $\J_\theta = \nabla_\theta h_\theta(0, 0) \in \R^{m \times m}$ and $\J_\varphi = \nabla_\theta h_\varphi(0, 0) \in \R^{p \times m}$. Let $\epsilon > 0$ and suppose $\rho^* + \epsilon < 1$. By \citet[Prop.~A.15]{bertsekasNonlinearProgramming1999} there exists a norm $\|.\|$ on $\R^m$ such that the induced matrix norm on $\R^{m \times m}$ satisfy:
\begin{equation}
\|\Id + J_\theta\| \leq \rho^* + \frac{\epsilon}{2}\,.
\end{equation}
On the contrary the norm on $\R^p$ can be chosen arbitrarily.

The extension of these norms to $\R^{m} \times \R^p$ is chosen such that $\|\theta, \varphi\| = \|\theta\| + \|\varphi\|$ for simplicity (but this is without loss of generality).
In this proof, we denote the $d$-dimensional balls by $B_d(x, r) = \{y \in \R^d\ |\ \|x - y\| \leq r\}$ with $x \in \R^d$, $r>0$.
\begin{itemize}
    \item Let $\J = \nabla_\theta h(0, 0) \in \R^{(m+p) \times m}$. 
    
    We first show that, for all $\eta > 0$ there exists $\tau \geq \delta > 0$ such that, 
    \begin{equation}
    \forall (\theta, \varphi) \in B_{m+p}((0, 0), \delta):\ \|h(\theta, \varphi) - \J\theta\| \leq \eta\|\theta\|\,.    
    \end{equation}
    The interesting thing here is that we are completely getting rid of the dependence on $\varphi$, both in the linearization and in the bound.
    
    Let $\varphi \in B_p(0, \tau)$. Then, using that $ h(0, \varphi) = 0$, the Taylor development of $h(\theta, \varphi)$ w.r.t.~to $\theta$ yields:
    \begin{align}
    h(\theta, \varphi) &= \nabla_\theta h(0, \varphi)\theta + R(\theta, \varphi)\\
    &= J\theta + (\nabla_\theta h(0, \varphi) - \nabla_\theta h(0, 0))\theta + R(\theta, \varphi)\,.\\
    \end{align}
    We now deal with the last two terms.
    First the rest $R(\theta, \varphi)$. As $v$ is assumed to be continuously differentiable, there exists $c > 0$ constant (which depends on $\tau$) such that, for any $\theta \in \R^m$:
    \begin{equation}
    \forall \varphi \in B_p(0, \delta): \|R(\theta, \varphi)\| \leq c\|\theta\|^2\,,
    \end{equation}
    Hence, for $\theta \in B(0, \frac{\eta}{2c})$, we get that:
    \begin{equation}
    \forall \varphi \in B_p(0, \delta): \|R(\theta, \varphi)\| \leq \frac{\eta}{2}\|\theta\|\,.
    \end{equation}
    Concerning the other term, by continuity, $\nabla_\theta h(0, \varphi) - \nabla_\theta h(0, 0)$ goes to zero as $\varphi$ goes to zero.
    Hence, there exists $\delta > 0$, $\delta \leq \min(\tau, \frac{\eta}{2c})$ such that for any $\varphi \in B_p(0, \delta)$, $\|(\nabla_\theta h(0, \varphi) - \nabla_\theta h(0, 0))\theta\| \leq \frac{\eta}{2}\|\theta\|$.
    Combining the two bounds yields the desired result.
    \item We now apply the previous result with $\eta = \epsilon/2$. We first examine what this means for $(\theta_{t+1}, \varphi_{t+1})$ when $(\theta_t, \varphi_t)$ is in $B_{m+p}((0, 0), \delta)$. However, the neighborhood $B_{m+p}((0, 0), \delta)$ is not necessarily stable, so we will again restrict it afterwards. See the proof \cite[Thm.~4]{nagarajanGradientDescentGAN2017} for a more detailed discussion on this.
    Assume for now that $(\theta_t, \varphi_t) \in B_{m+p}((0, 0), \delta)$. Then,
\begin{align}
\|\theta_{t+1}\| &= \|(\Id + \J_\theta)\theta_t + (h_\theta(\theta_t, \varphi_t) - \J_\theta\theta_t)\|\\
&\leq \|(\Id + \J_\theta)\theta_t\| + \|(h_\theta(\theta_t, \varphi_t) - \J_\theta\theta_t)\|\\
&\leq (\rho^* + \epsilon)\|\theta_t\|\,.\\
\end{align}
Consider now the other coordinate $\varphi_{t+1}$, still under the assumption that $(\theta_t, \varphi_t) \in B_{m+p}((0, 0), \delta)$. Then,
\begin{align}
\|\varphi_{t+1} - \varphi_t\| &= \|\J_\varphi\theta_t + (h_\varphi(\theta_t, \varphi_t) - \J_\varphi\theta_t)\|\\
&\leq \|\J_\varphi\theta_t\| + \|(h_\varphi(\theta_t, \varphi_t) - \J_\varphi\theta_t)\|\\
&\leq (\|\J_\varphi\| + \frac{\epsilon}{2})\|\theta_t\|\,.\\
\end{align}
Now let $V = \{(\theta, \varphi) \in \R^m \times \R^p\ | (\theta, \varphi) \in B((0, 0), \delta),\, (1 + \frac{\|\J_\varphi\| + \frac{\epsilon}{2}}{1 - \rho^* - \epsilon})\|\theta\| + \|\varphi\| < \delta\}$ neighborhood of $(0, 0)$. We show, by induction, that if $(\theta_0, \varphi_0) \in V$, then the iterates stay in $B_{m+p}((0, 0), \delta)$.

Assume  $(\theta_0, \varphi_0) \in V$. By construction, $(\theta_0, \varphi_0) \in B_{m+p}((0, 0), \delta)$. Now assume that $(\theta_0, \varphi_0),(\theta_1, \varphi_1),\dots,(\theta_t, \varphi_t)$ are in $B_{m+p}((0, 0), \delta)$ for some $t \geq 0$. By what has been proven above, first, $\|\theta_{t+1}\| \leq (\rho^* + \epsilon)^{t+1} \|\theta_0\| \leq \|\theta_0\|$. Then,
\begin{align}
\|\varphi_{t+1}\| &\leq \|\varphi_0\| + \sum_{k=0}^t\|\varphi_{k+1} - \varphi_k\|\\
&\leq \|\varphi_0\| + (\|\J_\varphi\| + \frac{\epsilon}{2})\sum_{k=0}^t\|\theta_{k}\|\\
&\leq \|\varphi_0\| + (\|\J_\varphi\| + \frac{\epsilon}{2})\sum_{k=0}^t(\rho^* + \epsilon)^k\|\theta_{0}\|\\
&\leq \|\varphi_0\| + \frac{\|\J_\varphi\| + \frac{\epsilon}{2}}{1 - \rho^* - \epsilon}\|\theta_{0}\|\\
\end{align}
Hence, putting the two coordinates together,
\begin{align}
\|(\theta_{t+1}, \varphi_{t+1})\| &= \|\theta_{t+1}\| + \|\varphi_{t+1}\|\\
&\leq \|\varphi_0\| + \left(1 + \frac{\|\J_\varphi\| + \frac{\epsilon}{2}}{1 - \rho^* - \epsilon}\right)\|\theta_{0}\|\\
&< \delta\,.
\end{align}
by definition of $V$. Hence, $(\theta_{t+1}, \varphi_{t+1}) \in B_{m+p}((0, 0), \delta)$ which concludes the induction and the proof.
\end{itemize}

For the linear operator case, note that we can choose $\tau = +\infty$, $c = 0$, $\eta = 0$ and $\delta = +\infty$. Then we have $V = \R^m \times \R^p$.
\end{proof}

By a linear base change, we get the more practical corollary:
\begin{cor}\label{cor: local convergence with singularity}
Let $F : \R^d \rightarrow \R^d$ be twice continuously differentiable and $\omega^* \in \R^d$ be a fixed point. Assume that there exists $\delta > 0$ such that for all $\xi \in \Ker(\nabla F(\omega^*) - I_d) \cap B(0, \delta)$, $\omega^* + \xi$ is still a fixed point and that $\Ker(\nabla F(\omega^*) - I_d)^2 =  \Ker(\nabla F(\omega^*) - I_d)$. Define
\begin{equation}
\rho^* = \max\{|\lambda|\, |\, \lambda \in \Sp \nabla F(\omega^*),\, \lambda \neq 1\}\,,
\end{equation}
and assume $\rho^* < 1$. Consider the iterates $(\omega_t)_t$  built from $\omega_0 \in \R^d$ as:
\begin{equation}
\omega_{t+1} = F(\omega_t)\quad \forall t \geq 0\,.    
\end{equation}
Then, for all $\epsilon > 0$, for any $\omega_0$ in a neighborhood of $\omega^*$, the distance of the iterates $(\omega_t)^t$ to fixed points of $F$ decreases in $\bigO((\rho^*+\epsilon)^t)$.

Moreover, if $F$ is linear, we can take this neighborhood to be the whole space and $\epsilon = 0$.
\end{cor}
\begin{proof}
We consider the spaces $\Ker(\nabla F(\omega^*) - \lambda I_d)^{m_\lambda}$, $\lambda \in \Sp \nabla F(\omega^*)$ where $m_\lambda$ denotes the multiplicity of the eigenvalue $\lambda$ as root of the characteristic polynomial of $\nabla F(\omega^*)$. Then, we have,
$$\R^d = \bigoplus_{\lambda \in \Sp \nabla F(\omega^*)} \Ker(\nabla F(\omega^*) - \lambda I_d)^{m_\lambda}\,, $$
see \citet[Chap.~6]{laxLinearAlgebraIts2007} for instance.

Now, using that $\Ker(\nabla F(\omega^*) - I_d)^2 =  \Ker(\nabla F(\omega^*) - I_d)$, we have that $\Ker(\nabla F(\omega^*) - I_d)^{m_1} = \Ker(\nabla F(\omega^*) - I_d)$. Hence, the whole space can be decomposed as $\R^d = \Ker(\nabla F(\omega^*) - I_d) \oplus E$ where $E =  \bigoplus_{\lambda \in \Sp \nabla F(\omega^*)\setminus\{1\}} \Ker(\nabla F(\omega^*) - \lambda I_d)^{m_\lambda}$. Note that $E$ is stable by $\nabla F(\omega^*)$ and so $\rho(\nabla F(\omega^*)\vert_E) = \rho^*$ as defined in the statement of the theorem.
Denote by $P \in \R^d\times\R^d$ the (invertible) change of basis such that $\Ker (\nabla F(\omega^*) - I_d)$ is sent on the subspace $\R^m \times \{0\}^p$ and $E$ on the subspace $\{0\}^m \times \R^p$, where $m$ and $p$ are the respective dimensions of $\Ker (\nabla F(\omega^*) - I_d)$ and $E$. Then, we apply the previous theorem \cref{thm: local convergence lemma with singularity} with $h$ defined by,
$$(\theta, \varphi) + h(\theta, \varphi) = PF\left(\omega^* + P^{-1}(\theta, \varphi)\right)\,.$$
which concludes the proof.
\end{proof}
\begin{rmk}
In general the condition $\Ker(\nabla F(\omega^*) - I_d)^2 =  \Ker(\nabla F(\omega^*) - I_d)$ will be equivalent to $\Ker\nabla v(\omega^*)^2 =  \Ker \nabla v(\omega^*)$ where $v$ is the game vector field. We keep this remark informal but  we prove this for extragradient below as an example. Indeed, as seen with $1-SCLI$ in \cref{subsection: pSCLI framework} with \cref{eq: 1-SCLI rule}, $\nabla F(\omega^*)$ is of the form $\Id + \mathcal{N}(\nabla v(\omega^*)) \nabla v(\omega^*)$ where $\mathcal{N}$ is a polynomial. Hence, $(\nabla F(\omega^*) - I_d)^j = \mathcal{N}(\nabla v(\omega^*))^j \nabla v(\omega^*)^j$. Moreover, in practice, $\mathcal{N}(\nabla v(\omega^*))$ will be chosen --- e.g.~by the choice of the step-size --- to be non-singular. Hence,  $\Ker (\nabla F(\omega^*) - I_d)^j = \Ker \nabla v(\omega^*)^j$ and so $\Ker(\nabla F(\omega^*) - I_d)^2 =  \Ker(\nabla F(\omega^*) - I_d)$ will be equivalent to $\Ker\nabla v(\omega^*)^2 =  \Ker \nabla v(\omega^*)$. 
\end{rmk}
We now prove a lemma concerning extragradient which as a first step before apply \cref{cor: local convergence with singularity}. We could have proven this result for $k$-extrapolation methods but we focus on extragradient for simplicity.
\begin{lemma}\label{lemma: correctness eg}
Let $F_{2, \eta} : \omega \rightarrow \omega - \eta v(\omega - \eta v(\omega))$ denote the extragradient operator. Assume that $v$ is L-Lipschitz. Then, if $0 <\eta < \frac{1}{L}$, for $\omega^*$ stationary point of $v$, 
$$\Ker(\nabla F_{2, \eta}(\omega^*) - I_d) = \Ker \nabla v(\omega^*)\,,$$
and 
$$\Ker(\nabla F_{2, \eta}(\omega^*) - I_d)^2 =  \Ker(\nabla F_{2, \eta}(\omega^*) - I_d) \iff \Ker\nabla v(\omega^*)^2 =  \Ker \nabla v(\omega^*)\,.$$
\end{lemma}
\begin{proof}
We have $\nabla F_{2, \eta}(\omega^*) = I_d - \eta \nabla v(\omega^*)(I_d - \eta \nabla v(\omega^*))$ and so $\nabla F_{2, \eta}(\omega^*) - I_d = - \eta \nabla v(\omega^*)(I_d - \eta \nabla v(\omega^*))$. As $\nabla v(\omega^*)$ and $I_d - \eta \nabla v(\omega^*))$ commute, for $j \in \{1,  2\}$,
$$\left(\nabla F_{2, \eta}(\omega^*) - I_d\right)^j = \left(- \eta\left(I_d - \eta \nabla v(\omega^*)\right)\right)^j \nabla v(\omega^*)^j\,. $$
By the choice of $\eta$, $\eta(I_d - \eta \nabla v(\omega^*))$ is non-singular and so $\Ker (\nabla F_{2, \eta}(\omega^*) - I_d)^j = \Ker \nabla v(\omega^*)^j$ which yields the result.
\end{proof}
 The whole framework developed implies in particular that \cref{thm: rate of n-extrapolation} actually also yields convergence guarantees for extragradient on more general bilinear games than those considered in \cref{ex: highly adversarial game}.
 \begin{ex}[Bilinear game with potential singularity]\label{ex: singular bilinear game}
 A saddle-point problem of the form:
  \begin{equation}
\min_{x \in \R^{m}} \max_{y \in \R^{p}} x^TA y + b^T x + c^Ty 
 \end{equation}
 with $A \in \R^{m \times p}$ not null, $b \in \R^m$, $c \in \R^p$.
 \end{ex}
 \begin{cor}
 Consider the bilinear game of \cref{ex: singular bilinear game}. The iterates of extragradient with $\eta = (4 \sigma_{max}(A))$ converge globally to $\omega^*$ at a linear rate of $\bigO\big(\big(1 - \frac{1}{64}\frac{\tilde \sigma_{min}(A)^2}{{\sigma_{max}(A)^2}}\big)^t\big)$ where $\tilde \sigma_{min}(A)$ is the smallest non-zero singular value of $A$.
 \end{cor}
 \begin{proof}
 Let $\omega^*$ be a stationary point of the associated vector field $v$. Then, $\nabla v(\omega^*) = \begin{pmatrix}
 O & A\\
 -A^T & 0
 \end{pmatrix}$ which is skew-symmetric. 
 Note that if $\eta = (4 \sigma_{max}(A))^{-1}$, then $0 < \eta < L$ where $L$ is the Lipschitz constant of $v$.
 
 We check that $\Ker\nabla v(\omega^*)^2 =  \Ker \nabla v(\omega^*)$. Let $X \in \R^{m+p}$ such that $\nabla v(\omega^*)^2 X = 0$. As $\nabla v(\omega^*)$ is skew-symmetric, this is equivalent to $\nabla v(\omega^*)^T \nabla v(\omega^*) X = 0$ which implies that $\|\nabla v(\omega^*)\| = 0$ which implies our claim. By \cref{lemma: correctness eg}, this implies that $\Ker(\nabla F_{2, \eta}(\omega^*) - I_d)^2 =  \Ker(\nabla F_{2, \eta}(\omega^*) - I_d)$.
 Moreover, if $\xi \in \Ker(\nabla F(\omega^*) - I_d)$ then by \cref{lemma: correctness eg}, $\xi \in \Ker \nabla v(\omega^*)$ and so $v(\omega^* + \xi) = 0$ too. Hence the hypothesises of \cref{cor: local convergence with singularity} are satisfied. 
 Then, by our choice of $\eta$ and \cref{lemma: spectra operators}, 
 
\begin{align*}
\rho^* &= \max\{|\lambda|\, |\, \lambda \in \Sp \nabla F_{2, \eta}(\omega^*),\, \lambda \neq 1\}\\
&= \max\{|1 - \eta\lambda(1 - \eta \lambda)|\, |\, \lambda \in \Sp \nabla v(\omega^*),\, \lambda \neq 0\}\\
&= \max\{|1 - \eta\lambda(1 - \eta \lambda)|\, |\, \lambda = \pm i\sigma,\, \sigma^2 \in \Sp AA^T,\, \sigma \neq 0\}\,,
\end{align*}
by a similar reasoning as \cref{lemma: spectrum bilinear game} since $\Sp AA^T \setminus\{0\} = \Sp A^TA \setminus\{0\}$.

The result is now a consequence of the proof of \cref{thm: rate of n-extrapolation}.
 \end{proof}

\section{Improvement of global rate}\label{section: app improvement of rate}
In this section we study experimentally the importance of the term $\eta^2\gamma^2$ in the global rate of \Cref{thm: global rate eg}. For this we generate two player zero-sum random montone matrix games, that is to say saddle-point problems of the form
\begin{equation}
\min_{\omega_1 \in \R^{d_1}}\max_{\omega_2 \in \R^{d_2}} \begin{pmatrix}
\omega_1 & \omega_2
\end{pmatrix}\begin{pmatrix}
S_1 & A \\
-A^T & S_2 \\
\end{pmatrix}\begin{pmatrix}
\omega_1 \\ \omega_2
\end{pmatrix}\,,
\end{equation}
where $S_1$ and $S_2$ are symmetric semi-definite positive.
To generate a symmetric semi-definite positive matrix of dimension $m$, we first draw independently $m$ non-negative scalars 
$\lambda_1,\dots,\lambda_m$ according to the chi-squared law. Then, we draw an orthogonal matrix $O$ according to the uniform distribution over the orthogonal group. The result is $S = O^T \diag(\lambda_1,\dots,\lambda_m)O$. The coefficients of $A$ are chosen independently according to a normal law $\mathcal{N}(0, 1)$.

To study how the use of Tseng's error bound improves the standard rate which uses the strong monotonicty only, we compute, for each such matrix game, the ratio $\frac{\eta\mu}{\eta\mu + \frac{7}{16} \eta^2\gamma^2}$ with $\eta = (4L)^{-1}$ (with the same notations as \Cref{thm: global rate eg}). This ratio lies between 0 and 1: if it close to $0$, it means that $\eta^2\gamma^2$ is much bigger than $\eta\mu$ so that our new convergence rate improves the previous one a lot, while if its near 1, it means that $\eta^2\gamma^2$ is much smaller than $\eta\mu$ and so that our new result does not improve much.

We realize two sets of graphics, each time keeping a different parameter fixed. These histograms are constructed from $N = 500$ samples.

\newcounter{N}
\setcounter{N}{500}

\newcommand{\drawImprovementPlot}[3]{
\begin{subfigure}{0.45\linewidth}
\begin{tikzpicture}[]
  \begin{axis}[
    width=7cm,
    height=7cm,
    ybar,
    xmin = 0, xmax=1,
    ymin = 0, ymax=\value{N},
    ylabel=Frequence,
    xlabel=Ratio,
    xtick = {0,0.2,0.4,0.6,0.8,1},
    yticklabel={\pgfmathparse{(\tick/\value{N})*100}\pgfmathprintnumber{\pgfmathresult}\%},
    yticklabel style={
      /pgf/number format/.cd,
      fixed, precision=0,
      /tikz/.cd
    },
    ]
    \addplot+ [hist={bins=50}]
    table [col sep=comma, y=#1] {gamma.csv};
  \end{axis}
\end{tikzpicture}
\subcaption{$d_1 = #2$ and $d_2 = #3$}
\end{subfigure}
}

\begin{figure}
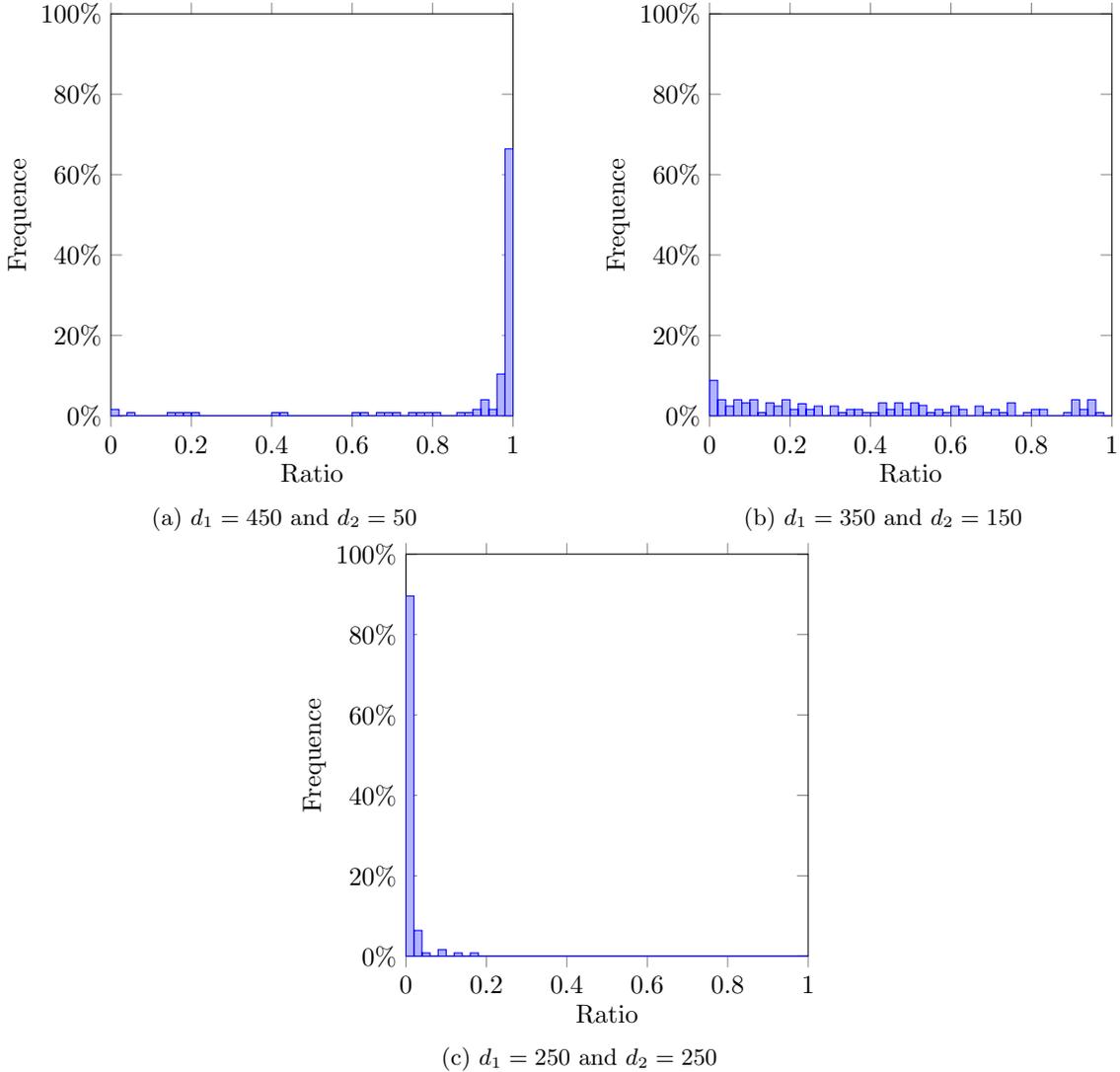

    \centering
    \drawImprovementPlot{450vs50}{450}{50}
\drawImprovementPlot{350vs150}{350}{150}
\drawImprovementPlot{250vs250}{250}{250}
    \caption{We keep $d_1 + d_2 = 500$ fixed and vary the balance between the two players.}
    \label{fig: graphics one}
\end{figure}
\begin{figure}
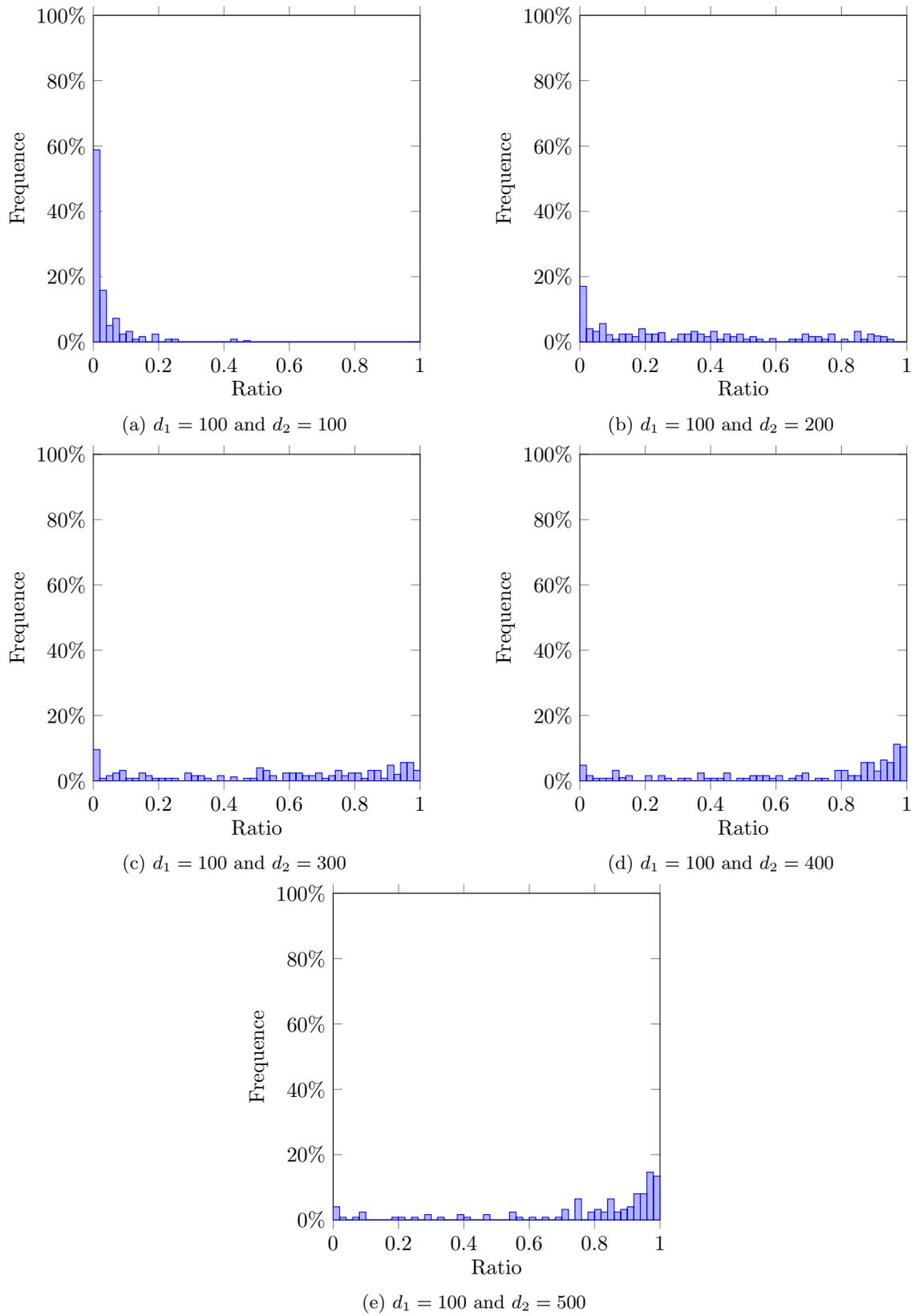

    \centering
\drawImprovementPlot{100vs100}{100}{100}
\drawImprovementPlot{200vs100}{100}{200}
\drawImprovementPlot{300vs100}{100}{300}
\drawImprovementPlot{400vs100}{100}{400}
\drawImprovementPlot{500vs100}{100}{500}
    \caption{We keep $d_1 = 100$ fixed and make $d_2$ vary.}
    \label{fig:my_label}
\end{figure}
What observe is that, as soon as none of the dimensions are too small, our new rate improves the previous one in many situations. This is greatly amplified if $d_1$ and $d_2$ are similar.
\end{document}